\documentclass{article}

\usepackage{microtype}
\usepackage{graphicx}
\usepackage{subfigure}
\usepackage{booktabs} 
\usepackage{epstopdf}

\usepackage{hyperref}

\usepackage[accepted]{icml2020}

\usepackage[utf8]{inputenc}
\usepackage{amsmath, mathrsfs, mathtools, amsfonts, amsthm}
\usepackage{color}
\allowdisplaybreaks

\newtheorem{proposition}{Proposition}

\icmltitlerunning{Interpretable Off-Policy Evaluation by Highlighting Influential Transitions}

\begin{document}

\twocolumn[
\icmltitle{Interpretable Off-Policy Evaluation in Reinforcement Learning \texorpdfstring{\\}{} by Highlighting Influential Transitions}

\begin{icmlauthorlist}
\icmlauthor{Omer Gottesman}{Har}
\icmlauthor{Joseph Futoma}{Har}
\icmlauthor{Yao Liu}{Sta}
\icmlauthor{Sonali Parbhoo}{Har}
\icmlauthor{Leo Anthony Celi}{MIT}
\icmlauthor{Emma Brunskill}{Sta}
\icmlauthor{Finale Doshi-Velez}{Har}
\end{icmlauthorlist}

\icmlaffiliation{Har}{Harvard University}
\icmlaffiliation{MIT}{MIT}
\icmlaffiliation{Sta}{Stanford University}

\icmlcorrespondingauthor{Omer Gottesman}{gottesman@fas.harvard.edu}

\icmlkeywords{off-policy evaluation, interpretability}

\vskip 0.3in
]

\printAffiliationsAndNotice{}

\begin{abstract}
    Off-policy evaluation in reinforcement learning offers the chance of using observational data to improve future outcomes in domains such as healthcare and education, but safe deployment in high stakes settings requires ways of assessing its validity. Traditional measures such as confidence intervals may be insufficient due to noise, limited data and confounding. In this paper we develop a method that could serve as a hybrid human-AI system, to enable human experts to analyze the validity of policy evaluation estimates. This is accomplished by highlighting observations in the data whose removal will have a large effect on the OPE estimate, and formulating a set of rules for choosing which ones to present to domain experts for validation. We develop methods to compute exactly the influence functions for fitted Q-evaluation with two different function classes: kernel-based and linear least squares, as well as importance sampling methods. Experiments on medical simulations and real-world intensive care unit data demonstrate that our method can be used to identify limitations in the evaluation process and make evaluation more robust.
\end{abstract}

\section{Introduction}

Within reinforcement learning (RL), off-policy evaluation (OPE) is the task of estimating the value of a given evaluation policy, using data collected by interaction with the environment under a different behavior policy \citep{sutton2018reinforcement, precup2000eligibility}. OPE is particularly valuable when interaction and experimentation with the environment is expensive, risky, or unethical---for example, in healthcare or with self-driving cars. However, despite recent interest and progress, state-of-the-art OPE methods still often fail to differentiate between obviously good and obviously bad policies, e.g. in healthcare \citep{gottesman2018evaluating}.

Most of the OPE literature focuses on sub-problems such as improving asymptotic sample efficiency or bounding the error on OPE estimators for the value of a policy. However, while these bounds are theoretically sound, they are often too conservative to be useful in practice (though see e.g.~\citet{thomas2019} for an exception). This is not surprising, as there is a theoretical limit to the statistical information contained in a given dataset, no matter which estimation technique is used. Furthermore, many of the common assumptions underlying these theoretical guarantees are usually not met in practice: observational healthcare data, for example, often contains many unobserved confounders \citep{gottesman2019guidelines}.

Given the limitations of OPE, we argue that in high stakes scenarios domain experts should be integrated into the evaluation process in order to provide useful actionable results. For example, senior clinicians may be able to provide insights that reduce our uncertainty of our value estimates. In this light, the explicit integration of expert knowledge into the OPE pipeline is a natural way for researchers to receive feedback and continually update their policies until one can make a responsible decision about whether to pursue gathering prospective data.

The question is then what information can humans provide that might help assess and potentially improve our confidence in an OPE estimate? In this work, we consider how human input could improve our confidence in the recently proposed OPE estimator, fitted Q-evaluation (FQE) \cite{le2019batch}, as well as importance sampling (IS) methods.  We develop an efficient approach to identify the most influential transitions in a batch of observational data, that is, transitions whose removal would have large effects on the OPE estimate. By presenting these influential transitions to a domain expert and verifying that they are indeed representative of the data, we can increase our confidence that our estimated evaluation policy value is not dependent on outliers, confounded observations, or measurement errors. The main contributions of this work are:

\begin{itemize}
    \item \emph{Conceptual}: We develop a framework for using influence functions to interpret OPE, and discuss the types of questions which can be shared with domain experts to use their expertise in debugging OPE.
    \item \emph{Technical}:  We develop computationally efficient algorithms to compute the exact influence functions for several IS estimators as well as two broad function classes for FQE: kernel-based functions and linear functions.
    \item \emph{Empirical}: We demonstrate the potential benefits of influence analysis for interpreting OPE on a cancer simulator, and present results of analysis together with practicing clinicians of OPE for management of acute hypotension from a real intensive care unit (ICU) dataset.
\end{itemize}

\section{Related work}

The OPE problem in RL has been studied extensively.  Works fall into two main categories: importance sampling (e.g. \citet{precup2000eligibility, jiang2015doubly}) and model-based (often referred to as the direct method), which can be further subdivided into modeling the environment dynamics (e.g. \citet{hanna2017bootstrapping, gottesman2019combining}), and directly modeling the value function (e.g. \citet{le2019batch}). Some of these works provide bounds on the estimation errors (e.g. \citet{thomas2015high, dann2018policy}). We emphasize, however, that for most real-world applications these bounds are either too conservative to be useful or rely on assumptions which are usually violated.

While there has been considerable recent progress in interpretable machine learning and machine learning with humans in the loop (e.g. \citet{tamuz2011adaptively, lage2018human}), to our knowledge, there has been little work that considers human interaction in the context of OPE. \citet{oberst2019counterfactual} proposed framing the OPE problem as a structural causal model, which enabled them to identify trajectories where the predicted counterfactual trajectories under an evaluation policy differs substantially from the observed data collected under the behavior policy. However, that work does not give guidance on what part of the trajectory might require closer scrutiny, nor can it use human input for additional refinement. 

Finally, the notion of influence that we use throughout this work has a long history in statistics as a technique for evaluating the robustness of estimators \citep{cook1980characterizations}.  Recently, an approximate version of influence for complex black-box models was presented in \citet{koh2017understanding}, and they demonstrated how influence functions can make machine learning methods more interpretable. In the context of optimal control and RL, influence functions were first introduced by \citet{munos2002variable} to aid in online optimization of policies. However, their definition of influence as a change in the value function caused by perturbations of the reward at a specific state is quite different from ours. 

\section{Background}

\paragraph{Notation}

A Markov Decision Process (MDP) is a tuple $\langle \mathcal{X}, \mathcal{A}, P_T, P_R, P_0, \gamma \rangle$, where $\mathcal{X}$, $\mathcal{A}$ and $\gamma$ are the state space, action space, and the discount factor, respectively. The next state transition and reward distributions are given by $P_T(\cdot | x, a)$ and $P_R(\cdot | x, a)$ respectively, and $P_0(x)$ is the initial state distribution. The state and action spaces could be either discrete or continuous, and the transition and reward functions may be either stochastic or deterministic. 

A dataset is composed of a set of $N$ observed transitions $\mathcal{D} =  \{ (x^{(n)}, a^{(n)} ,r^{(n)} ,x'^{(n)}) \}_{n=1}^N$, and we use $\tau^{(n)}$ to denote a single transition. The subset $\mathcal{D}_0 \subseteq \mathcal{D}$ denotes initial transitions from which $P_0$ can be estimated. Note that although we treat all data points as observed transitions, in most practical applications data is collected in the form of trajectories rather than individual transitions.

A policy is a function $\pi : (\mathcal{X}, \mathcal{A}) \rightarrow \left[0, 1 \right]$ that gives the probability of taking each action at a given state $(\sum_{a \in \mathcal{A}} \pi(a|x) = 1)$.  The value of a policy is the expected return collected by following the policy, $v^\pi \coloneqq \mathrm{E} [ g_T | a_t \sim \pi]$, where expectations are taken with respect to the MDP and $g_T \coloneqq \sum_{t=0}^T \gamma^t r_t$ denotes the total trajectory return (sum of discounted rewards). The state-action value function $q^{\pi}(x, a)$ is the expected return for taking action $a$ at state $x$, and afterwards following $\pi$ in selecting future actions. The goal of off-policy evaluation is to estimate the value of an \emph{evaluation} policy, $\pi_e$, using data collected under a different \emph{behavior} policy, $\pi_b$. In this work, we are only interested in estimating $v^{\pi_e}$ and $q^{\pi_e}$, and will therefore drop the superscript for brevity. We will also limit ourselves to deterministic evaluation policies.

For the purpose of kernel-based value function approximation, we define a distance metric, $d((x^{(i)}, a^{(i)}),(x^{(j)}, a^{(j)}))$ over $\mathcal{X} \times \mathcal{A}$. In this work, for discrete action spaces, we will assume $d((x^{(i)}, a^{(i)}),(x^{(j)}, a^{(j)})) = \infty$ when $a^{(i)} \neq a^{(j)}$, but this is not required for any of the derivations.

\paragraph{Fitted Q-Evaluation}

Fitted Q-Evaluation \citep{le2019batch} models the q-function of $\pi_e$ and can be thought of as dynamic programming on an observational dataset to compute the value of a given evaluation policy. It is similar to the more well-known fitted Q-iteration method (FQI) \citep{ernst2005tree}, except it is performed offline on observational data, and the target is used for evaluation of a given policy rather than for optimization. FQE performs a sequence of supervised learning steps where the inputs are state-action pairs, and the targets at each iteration are given by $y_i(x, a) = r + \gamma \hat{q}_{i-1}(x', \pi_e(x'))$, where $\hat{q}_{i-1}(x, a)$ is the estimator (from a function class $\mathcal{F}$) that best estimates $y_{i-1}(x, a)$. For more information, see \citet{le2019batch}.

\paragraph{Importance sampling}

A popular class of OPE estimators consists of IS methods. These methods estimate the value of a policy by taking a sample average of trajectories returns, properly weighted to account for the difference between $\pi_b$ and $\pi_e$. The standard IS estimator is unbiased but has high variance, and there are many variants of this estimator which trade of bias and variance. For more information see \citep{precup2000eligibility, jiang2015doubly, thomas2016data}.

\section{OPE diagnostics using influence functions}

\subsection{Definition of the influence}

We aim to make OPE interpretable and easy to debug by identifying transitions in the data which are highly influential on the estimated policy value. We define the \emph{total influence} of transition $\tau^{(j)}$ as the change in the value estimate if $\tau^{(j)}$ was removed:
\begin{equation}
    I_j \equiv \hat{v}_{-j} - \hat{v},
\end{equation}
where $\hat{v}_{-j}$ is the value estimate using the same dataset after removal of $\tau^{(j)}$. In general, for any function of the data $f(\mathcal{D})$ we will use $f(\mathcal{D}_{-j}) \equiv f_{-j}$ to denote the value of $f$ computed for the dataset after removal of $\tau^{(j)}$.

Another quantity of interest is the change in the estimated value of $q(x^{(i)}, a^{(i)})$ as a result of removing $\tau^{(j)}$, which we call the \emph{individual influence}:
\begin{equation}
    I_{i, j} \equiv \hat{q}_{-j}(x^{(i)}, a^{(i)}) - \hat{q}(x^{(i)}, a^{(i)}).
\end{equation}
The total influence of $\tau^{(j)}$ can be computed by averaging its individual influences over the set $\mathcal{D}^*_0$ of all initial state-action transitions in which $a=\pi_e(x)$: 
\begin{equation}
    I_j = \frac{1}{|\mathcal{D}^*_0|} \sum_{i \in \mathcal{D}^*_0} I_{i, j}.
\end{equation}
As we are interested in the robustness of our evaluation, we can normalize the absolute value of the influence of $\tau^{(j)}$ by the estimated value of the policy to provide a more intuitive notion of overall importance:
\begin{equation}
    \tilde{I}_j \equiv \frac{|I_j|}{|\hat{v}|}.
\end{equation}

\subsection{Diagnosing OPE estimation}

With the above definitions of influence functions, we now formulate and discuss guidelines for diagnosing the OPE process for potential problems. 

\paragraph{No influential transitions: OPE appears reliable.}  
As a first diagnostic, we check that none of the transitions influence the OPE estimate by more than a specified influence threshold $\tilde{I}_C$, i.e. for all $j$ we have $\tilde{I}_j \leq \tilde{I}_C$. In such a case we would output that, to the extent that low influences suggests the OPE is stable, the evaluation appears reliable. That said, we emphasize that our proposed method for evaluating OPE methods is not exhaustive, and there could be many other ways in which OPE could fail.

\paragraph{Influential transitions: a human can help.}
When there are several influential transitions in the data (defined as transitions whose influence is larger than $\tilde{I}_C$), we present them to domain experts to determine whether they are representative, that is, taking action $a$ in state $x$ is likely to result in transition to $x'$.  If the domain experts can validate all influential transitions, we can still have some confidence in the validity of the OPE. If any influential transitions are flagged as unrepresentative or artefacts, we have several options: (1) Declare the OPE as unreliable; (2) Remove the suspect influential transitions from the data and recompute the OPE; (3) Caveat the OPE results as valid only for a subset of initial states that do not rely on that problematic transition.

In situations where there is a large number of influential transitions, manual review by experts may be infeasible. As such, it is necessary to present as few transitions as possible while still presenting enough to ensure that any potential artefacts in the data and/or the OPE process are accounted for. In practice, we find it is common to observe a sequence of influential transitions where removing any single transition has the same effect as removing the entire sequence. An example of this is shown schematically in Figure \ref{fig:influential_sequence}. An entire sequence marked in blue and red leads to a region of high reward, and so all transitions in that sequence will have high influence. The whole influential sequence appears very different from the rest of the data, and a domain expert might flag it as an outlier to be removed. However, we can present the expert with only the red transition and capture the influence of the blue transitions as well, reducing the number of suspect examples to be manually reviewed.

\paragraph{Influential transitions: policy is unevaluatable.}  When an influential transition, $\tau^{(j)}$, has no nearest neighbors to $(x'^{(j)}, \pi_e(x'^{(j)}))$, we can determine that the evaluation policy cannot be evaluated, even without review by a domain expert. This claim is a result of the fact that such a situation represents reliance of the OPE on transitions for which there is no overlap between the actions observed in the data and the evaluation policy. However, while the evaluation policy is not evaluatable, the influential ``dead-end'' transitions may still inform experts of what data is required for evaluation to be feasible.

It should be noted that the applicability of the diagnostics methods discussed above may change depending on whether the FQE function class is parametric or nonparametric. All function classes lend themselves to highlighting of highly influential transitions. However, the notion of stringing together sequences of neighbors, or looking for red flags in the form of influential transitions with no neighbors to their $(x', \pi_e(x'))$ state action pairs only makes sense for nonparametric models. In the case of parametric models, the notion of neighbors is less important as the influence of removing a transition manifests as a change to the learned parameters which affects the value estimates for the entire domain simultaneously. In contrast, for nonparametric methods, removing a transition locally changes the value of neighboring transitions and propagates through the entire domain through the sequential nature of the environment. While we derive efficient ways to compute the influence for both parametric and nonparametric function classes, in the empirical section of this paper we present results for nonparametric kernel-based estimators to demonstrate all diagnostics.

\begin{figure}[t]
\centering
\includegraphics[width=0.3\textwidth]{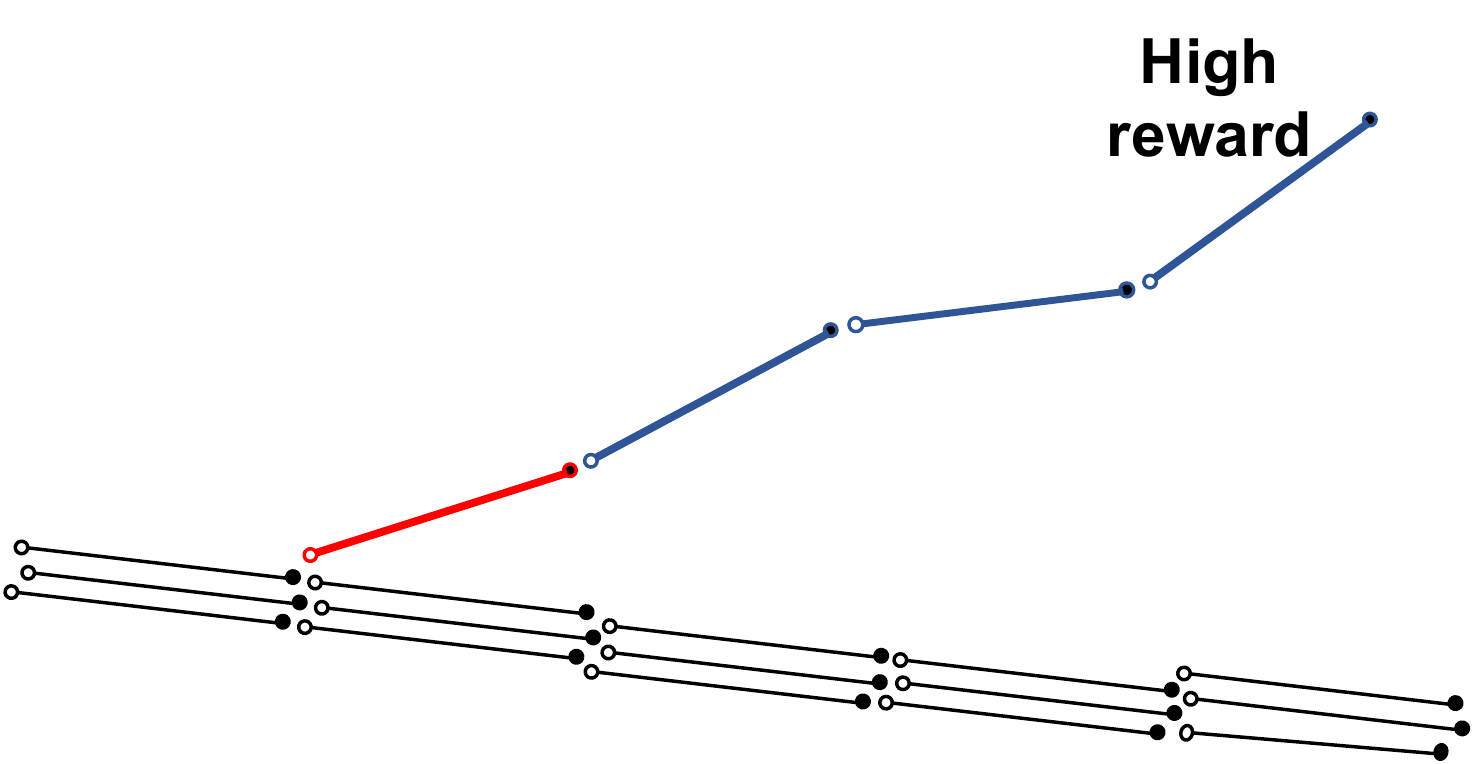}
\caption{\textbf{Schematic of an influential sequence.} All transitions in the sequence leading to a high reward have high influence, but flagging just the red transition for inspection will capture the influence of the blue ones as well.}\label{fig:influential_sequence}
\end{figure}

\subsection{Influence analysis for importance sampling}

The approach of using influence analysis to to asses the validity of the OPE can be naturally extended to IS methods, with a few small changes. Most IS methods use entire trajectories rather than individual transitions as their basic data input, and therefore for IS we would compute the influence of trajectories rather than transitions. This also implies that we cannot identify obvious unevaluateble datasets as described in the previous section. Last, it should be noted that for IS the influence is determined not only by the return of a trajectory, but is also strongly determined by the weights, which may grow exponentially with the horizon.

\section{Efficient computation of influence functions}
\label{sec:computation_of_influence}

A key technical challenge in performing the proposed influence analysis in OPE is computing the influences efficiently. The brute-force approach of removing a transition and recomputing the OPE estimate is clearly infeasible for all but tiny problems, as it requires refitting $N$ models. The computation of influences in RL is also significantly more challenging than in static one-step prediction tasks, as a change in the value of one state has a ripple effect on all other states that are possible to reach from it. We describe computationally efficient methods to compute the influence functions in two classes of FQE: kernel-based, and linear least squares, as well as several popular IS estimators. Unlike previous works (e.g. \citep{koh2017understanding}) that approximate the influence function for a broad class of black-box functions, we provide closed-form, analytic solutions for the exact influence function for a broad range of OPE methods.

\subsection{Kernel-Based FQE}
\label{sec:kernel_based_fqe}

In kernel based FQE, the function class we choose for estimating the value function of $\pi_e$ at a point in state-action space is based on similar observations within that space. For simplicity, in the main body of this work we estimate the value function as an average of all its neighbors within a ball of radius $R$, i.e.

\begin{align}
    \hat{q}(x, a) = \frac{1}{N_{(x, a)}} \sum_{i} \hat{q}(x^{(i)}, a^{(i)})
\end{align}

where the summation is performed over all $(x^{(i)}, a^{(i)})$ such that $d((x^{(i)}, a^{(i)}),(x, a)) < R$ and $N_{(x, a)}$ is the number of such points. Extension to general kernel functions is straightforward. We introduce a matrix formulation for performing FQE which allows for efficient computation of the influence functions.

\paragraph{Matrix formulation of nearest-neighbors based FQE.}

We define $\Delta_{i j}$ as the event that the starting state-action of $\tau^{(j)}$ is a neighbor of the starting state-action of $\tau^{(i)}$, i.e. $d((x^{(i)}, a^{(i)}), (x^{(j)}, a^{(j)})) < R$. Similarly, we define $\Delta_{i' j}$ as the event that the starting state-action of $\tau^{(j)}$ is a neighbor of the next-state and corresponding $\pi_e$ action of $\tau^{(i)}$, i.e. $d((x'^{(i)}, \pi(x'^{(i)})), (x^{(j)}, a^{(j)})) < R$. We also define the counts for numbers of neighbors of transitions as $N_i = \sum_{j=1}^N \mathbb{I} (\Delta_{i j})$ and $N_{i'} = \sum_{j=1}^N \mathbb{I} (\Delta_{i' j})$, where $\mathbb{I}(e)$ is the indicator function.

To perform nearest-neighbors FQE using matrix multiplications, we first construct two nearest-neighbors matrices: one for the neighbors of all state-action pairs, and one for the neighbors of all state-action pairs with pairs of next-states and subsequent actions under $\pi_e$. Formally:
\begin{equation}
    \mathbf{M}_{i j} = \frac{\mathbb{I} (\Delta_{i j})}{N_i}; \quad \mathbf{M}'_{i j} = \frac{\mathbb{I} (\Delta_{i' j})}{N_{i'}}.
\end{equation}

The $N \times N$ matrices $\mathbf{M}$ and $\mathbf{M}'$ can be easily computed from the data, and are used to compute the value function for all state-action pairs using the following proposition, the proof of which is given in Appendix \ref{appendix:proof_matrix_fqe}.

\begin{proposition}
\label{prop:matrix_fqe}
For all transitions in the dataset, the values for corresponding state-action pairs are given by
\begin{align}
    \mathbf{\hat{q}}'_t &= \left( \sum_{t'=1}^t \gamma^{t'-1} \mathbf{M}'^{t'} \right) \mathbf{r} \equiv \mathbf{\Phi}'_t \mathbf{r} \label{eq:q_prime_estimate_main} \\
    \mathbf{\hat{q}}_t &= \mathbf{M} \left( \sum_{t'=1}^t \left( \gamma \mathbf{M}' \right)^{t'-1} \right) \mathbf{r} \equiv \mathbf{\Phi}_t \mathbf{r}. \label{eq:q_estimate}
\end{align}
where $\hat{q}'_{t, i}$ and $\hat{q}_{t, i}$ are the estimated policy values at $(x'^{(i)}, \pi_e(x'^{(i)}))$ and $(x^{(i)}, a^{(i)})$, respectively, for $\tau^{(i)}$.
\end{proposition}

In future derivations, we will drop the time dependence of $\mathbf{\Phi}$ and $\hat{\mathbf{q}}$ on $t$. This is justified when there are well defined ends of trajectories with no nearest neighbors (or equivalently, trajectories end in an absorbing state), and the number of iterations in the FQE is larger than the longest trajectory.

\paragraph{Influence function computation.}

Removal of a transition $\tau^{(j)}$ from the dataset can affect $\hat{q}_i$ in two ways. First, $\hat{q}_i$ is a mean over all of its neighbors, indexed by $k$, of $r^{(k)} + \gamma \hat{q}'_k$. Thus if $(x^{(j)}, a^{(j)})$ is one of the $M^{-1}_{ij}$ neighbors of $(x^{(i)}, a^{(j)})$, removing it from the dataset will change the value of $\hat{q}_i$ by $\frac{\hat{q}_{i} - \left( r^{(j)} + \gamma \hat{q}'_{j} \right)}{M_{ij}^{-1} - 1}$. The special case of $M^{-1}_{ij} = 1$ does not pose a problem in the denominator, as given that $i \neq j$ and every transition is a neighbor of itself, if $(x^{(j)}, a^{(j)})$ is a neighbor of $(x^{(i)}, a^{(i)})$, then $M^{-1}_{ij} \geq 2$.

The second way in which removing $\tau^{(j)}$ influences $\hat{q}_i$ is through its effect on intermediary transitions. Removal of $\tau^{(j)}$ changes the estimated value of $\hat{q}'_k$, of all $(x'^{(k)}, \pi_e(x'^{(k)}))$ that $(x^{(j)}, a^{(j)})$ is a neighbor of by $\frac{\hat{q}'_{k} - \left( r^{(j)} + \gamma \hat{q}'_{j} \right)}{M'^{-1}_{kj} - 1}$. Multiplying this difference by $\gamma$ yields the difference in $\hat{q}_{k}$ due to removal of $\tau^{(j)}$. A change in the value of $\hat{q}_{k}$ is identical in its effect on the value estimation to changing $r^{(k)}$, a change which is mediated to $\hat{q}_i$ through $\Phi_{ik}$. In the special case that $(x^{(j)}, a^{(j)})$ is the only neighbor of $(x'^{(k)}, \pi_e(x'^{(k)}))$, the value estimate $\hat{q}'_k$ changes from $\hat{q}_j$ to zero.

Combining the two ways in which removal of $\tau^{(j)}$ changes the estimated value $\hat{q}_i$ yields the individual influence:
\begin{align}
    I_{i, j} &= \mathbb{I}(\Delta_{ij}) \frac{\hat{q}_{i} - \left( r^{(j)} + \gamma \hat{q}'_{j} \right)}{M_{ij}^{-1} - 1} + \sum_{k:\Delta_{k'j}} I_{(i, j)}^{(k)},
\end{align}
where we define
\begin{align}
    I_{i, j}^{(k)} =
    \begin{cases} 
      \gamma \Phi_{ik} \frac{\hat{q}'_{k} - \left( r^{(j)} + \gamma \hat{q}'_j \right)}{M'^{-1}_{kj} - 1} & M'^{-1}_{kj} > 1 \\
      \gamma \Phi_{ik} \hat{q}_j & M'^{-1}_{kj} = 1.
   \end{cases}
\end{align}

\paragraph{Computational complexity.} The matrix formulation of kernel based FQE allows us to compute an individual influence in constant time, making influence analysis of the entire dataset possible in $\mathcal{O}(N |\mathcal{D}^*_0|)$ time. Furthermore, the sparsity of $\mathbf{M}$ and $\mathbf{M}'$ allows the FQE itself to be done in $\mathcal{O}(N^2 T)$. See Appendix \ref{appendix:kernel_fqe_complexity} for a full discussion.

\subsection{Linear Least Squares FQE}

\newcommand{\feature}{\pmb{\psi}}
\newcommand{\Feature}{\mathbf{\Psi}}
In linear least squares FQE, the policy value function $\hat{q}(x,a)$ is approximated by a linear function $\hat{q}(x,a) = \feature(x,a)^\top \mathbf{w}$ where $\feature(x,a)$ is a $D$-dimensional feature vector for a state-action pair. Let $\Feature \in \mathbb{R}^{N \times D}$ be the sample matrix of $\feature(x,a)$. Define vector $\feature_{\pi}(x) = \gamma \feature(x, \pi_e(x))$ and let $\Feature_p \in \mathbb{R}^{N \times D}$ be the sample matrix of $\feature_{\pi}(x')$. The least-squares solution of $\mathbf{w}$ is 
    $(\Feature^\top \Feature - \gamma \Feature^\top \Feature_p)^{-1}\Feature^\top \mathbf{r}$
(See Appendix \ref{prop:ls_fqe} for full derivation).

Let $\mathbf{w}_{-j}$ be the solution of linear least squares FQE after removing $\tau^{(j)}$, and $\Feature_{-j}$, $\mathbf{r}_{-j}$, and $\Feature_{p,-j}$ be the corresponding matrices and vectors without the $\tau^{(j)}$. Then, $\mathbf{w}_{-j} = (\Feature_{-j}^\top \Feature_{-j} - \gamma \Feature_{-j}^\top \Feature_{p,-j})^{-1}\Feature_{-j}^\top \mathbf{r}_{-j}$. The key challenge of computing the influence function is computing $\mathbf{w}_{-j}$ in an efficient manner that avoids recomputing a costly matrix inverse for each $j$. 
Let $\mathbf{C}_{-j} = (\Feature_{-j}^\top \Feature_{-j} - \gamma \Feature_{-j}^\top \Feature_{p,-j})$ and $\mathbf{C} = (\Feature^\top \Feature - \gamma \Feature^\top \Feature_p)$. We compute $\mathbf{w}_{-j}$ as follows:
\begin{align}
    \mathbf{B}_{j} &\leftarrow \mathbf{C}^{-1} + \frac{\mathbf{C}^{-1} \feature_j \feature_j^\top \mathbf{C}^{-1}}{1- \feature_{j}^{\top}\mathbf{C}^{-1}\feature_{j}}\\
    \left(\mathbf{C}_{-j}\right)^{-1} &\leftarrow \mathbf{B}_{j} - \frac{\gamma \mathbf{B}_{j} \feature_j \feature_{\pi,j}^\top \mathbf{B}_{j}}{1+ \gamma \feature_{\pi,j}^{\top}\mathbf{B}_{j}\feature_{j}} \\
    \mathbf{w}_{-j} &\leftarrow \left(\mathbf{C}_{-j}\right)^{-1}\left( \Feature^\top \mathbf{r} - r^{(j)} \feature_j \right) 
\end{align}
The proof of correctness is in Proposition \ref{prop:ls_fqe_influence} in Appendix \ref{appendix:ls_fqe}. The individual influence function is then simply:
\begin{align}
    I_{i,j} = \feature(s^{(i)},a^{(i)})^\top (\mathbf{w}_{-j} - \mathbf{w}).
\end{align}

\paragraph{Computational complexity.} The bottleneck of computing $\mathbf{w}_{-j}$ is the matrix multiplication of $D \times D$ matrices which takes at most $\mathcal{O}(D^{3})$. All the other matrix multiplications involving size $N$, e.g. $\Feature^\top \mathbf{r}$, do not depend on $j$ and could be cached from the original OPE. Thus, the overall complexity for computing $I_{i,j}$ for all $i$ and $j$ is $\mathcal{O}(ND^{3})$. Assuming $N>D$, the complexity of the original OPE algorithm is $\mathcal{O}(ND^2)$, where the bottleneck is computing $\Feature^\top \Feature$.

\subsection{Importance Sampling}

IS methods are essentially weighted averages over returns of trajectories, and therefore computing the total influence of a trajectory in a dataset can easily be performed in constant time, as long as certain values a cached. For example, the influence of the $j^{th}$ trajectory for standard IS is

\begin{equation}
    I_j = \frac{1}{N-1} \left( \hat{v} - w_{0:T}^{(j)} g_T^{(j)} \right),
\end{equation}

where $N$ is the number of trajectories, and $w_{0:T}^{(j)}$ and $g_T^{(j)}$ are the IS weight and return of the $j^{th}$ trajectory, respectively. In Appendix \ref{appendix:is_computation} we present the derivation of the influence for IS, WIS, PDIS, DR and WDR estimators.

\section{Illustration of influence functions in a sequential setting}
\label{sec:intuition}

We now demonstrate and give intuition for how the influence behaves in an RL setting. For the demonstrations and experiments presented throughout the rest of the paper we use the kernel-based FQE method.

Several factors determine the influence of a transition. For transitions to be influential they must have actions which are possible under the evaluation policy and form links in sequences which result in returns different than the expected value. Furthermore, transitions will be more influential the less neighbors they have.

To demonstrate this intuition we present in Figure \ref{fig:intuition} trajectories from a 2D continuous navigational domain \footnote{Code for reproducing the results in this paper can be found at https://github.com/dtak/interpretable\_ope\_public.git}. The agent starts at the origin and takes noisy steps of length $1$ at $45^{\circ}$ to the axes. The reward for a given transition is a function of the state and has the shape of a Gaussian centered along the approximate path of the agent, represented as the background heat map in Figure \ref{fig:intuition} (top), where observed transitions are drawn as black line segments. Because distances for the FQE are computed in the state-action space, in this example all actions in the data are the same to allow for distances to be visualized in 2D.

To illustrate how influence is larger for transitions with few neighbors, we removed most of the transitions in two regions (denoted II and III), and compared the distribution of  influences in these regions with influences in a data dense region (denoted I). Figure \ref{fig:intuition} (bottom)  shows the distribution over 200 experiments (in each experiment, new data is generated) of the influences of transitions in the different regions. The influence is much higher for transitions in sparse regions with few neighbors, as can be seen by comparing the distributions in regions I and II.  This is a desired property, as in analysis of the OPE process, we'd like to be able to present domain experts with transitions that have few neighbors where the sampling variance of a particular transition could have large effect on evaluation.

In region III, despite the fact that the observations examined also have very few neighbors, their influence is extremely low, as they don't lead to any regions where rewards are gained by the agent.

\begin{figure}[t]
\centering
\subfigure{\includegraphics[width=0.35\textwidth]{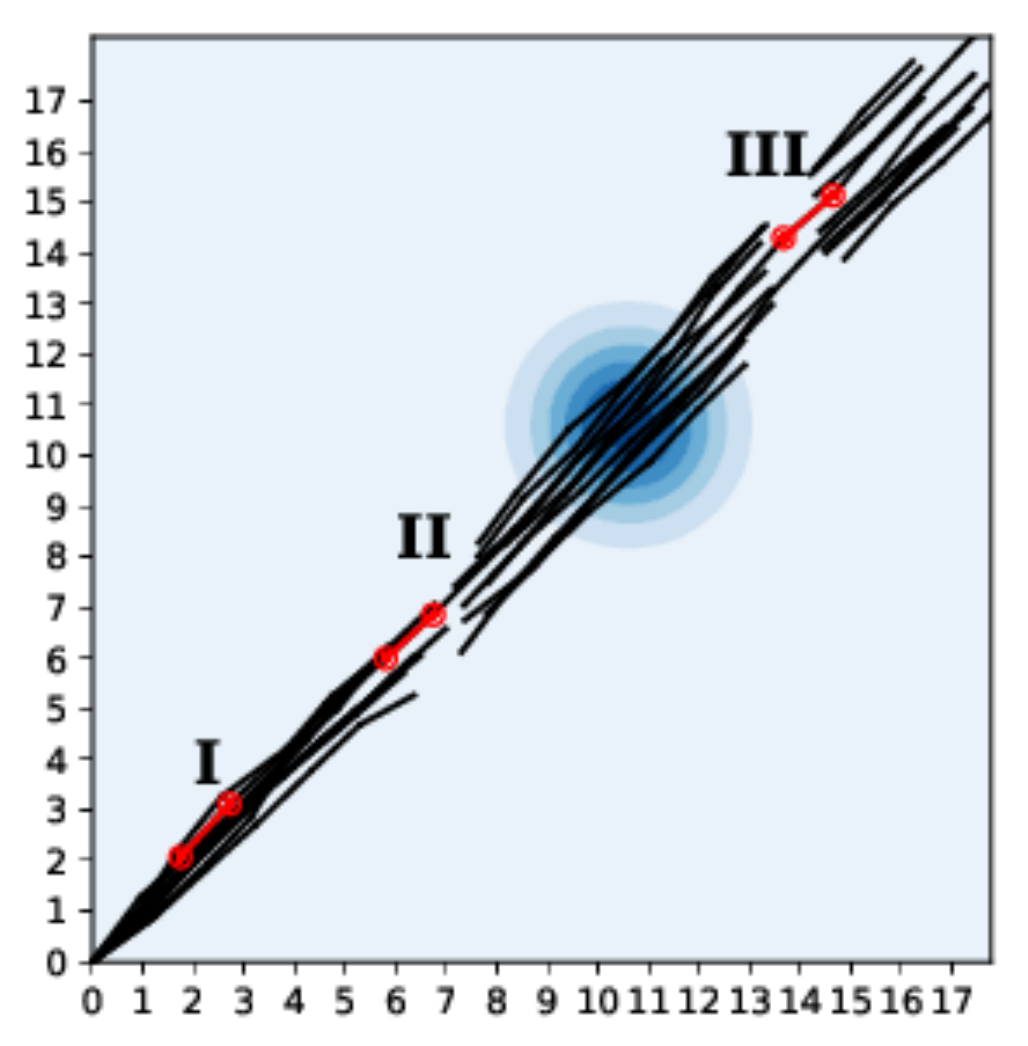}}
\subfigure{\includegraphics[width=0.35\textwidth]{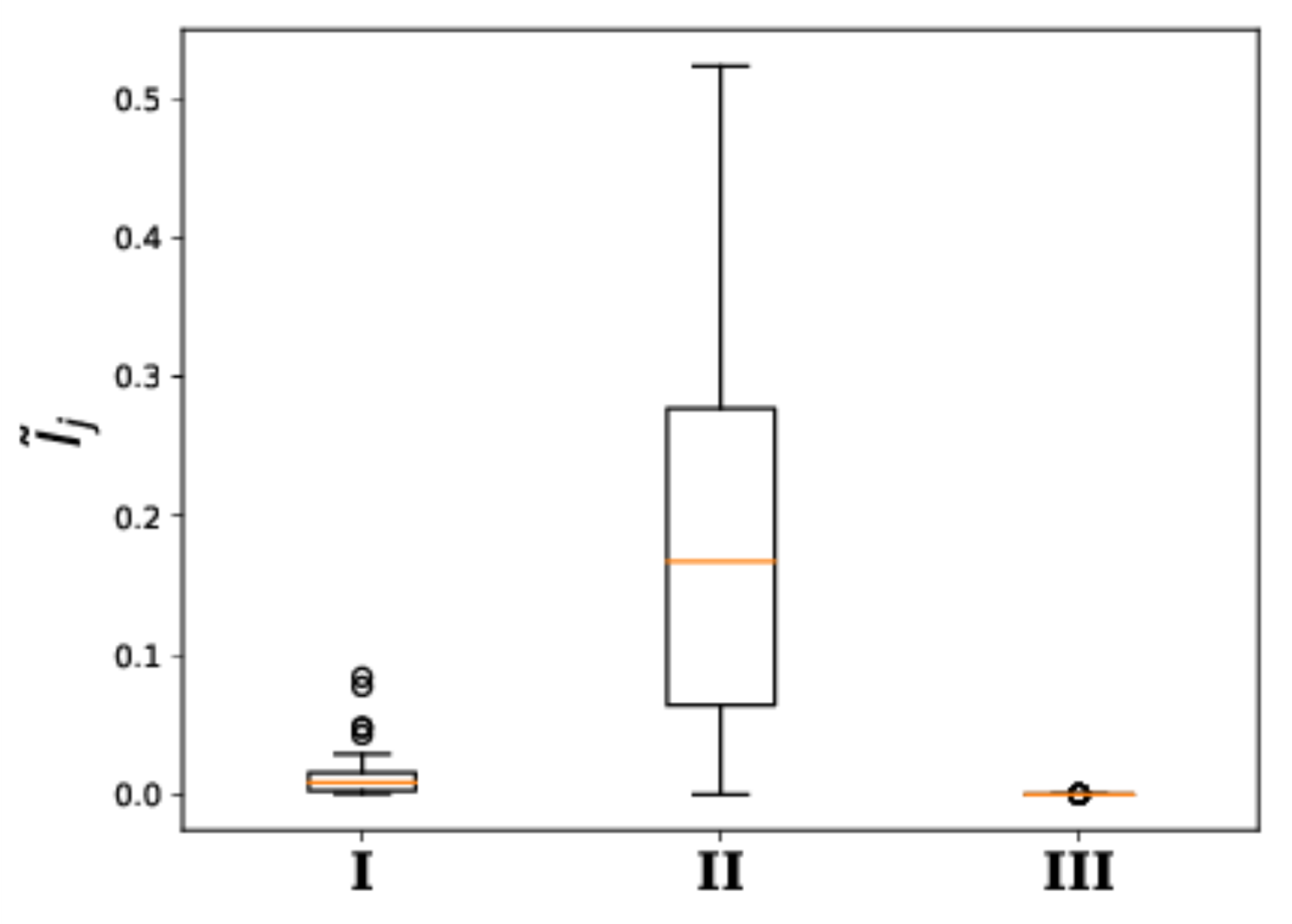}}
\caption{\textbf{Conceptual demonstration on a 2D domain.} For transitions in the data to have high influence, they must agree with the evaluation policy and lead to rewarding regions in the state-action space. Additionally, the influence of transitions decreases with the number of their close neighbors.}
\label{fig:intuition}
\end{figure}

\section{Experiments}

\subsection{Medical cancer simulator}
\label{sec:exp_cancer}

\begin{figure}[t]
    \centering
    \subfigure[No influential transitions]{\includegraphics[width=0.4\textwidth]{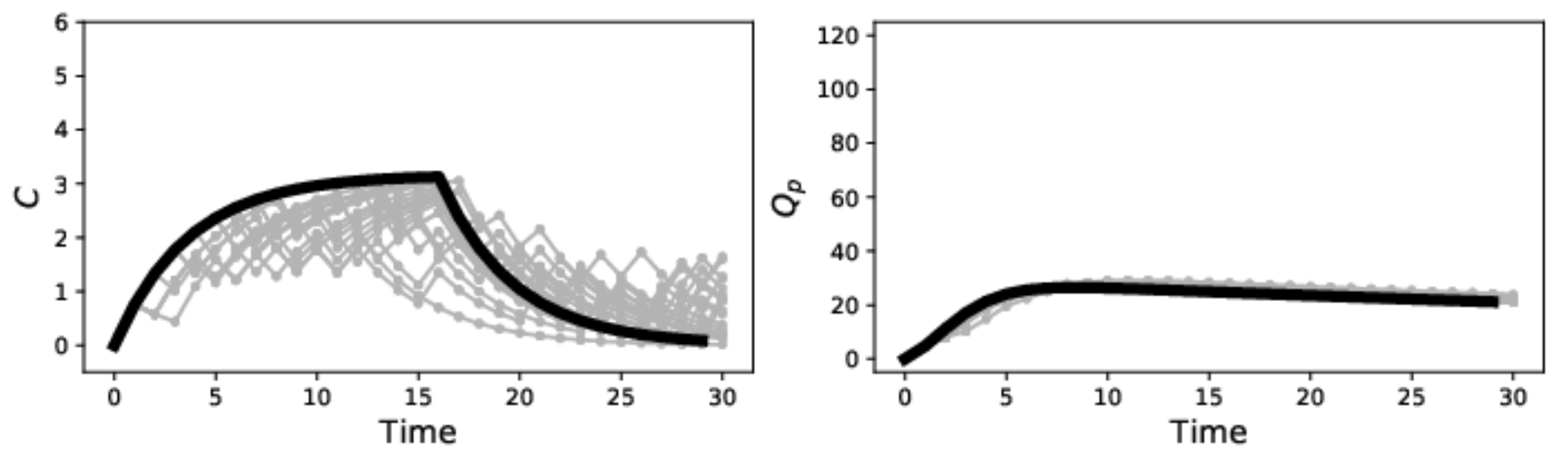}}
    \subfigure[Dead end sequence]{\includegraphics[width=0.4\textwidth]{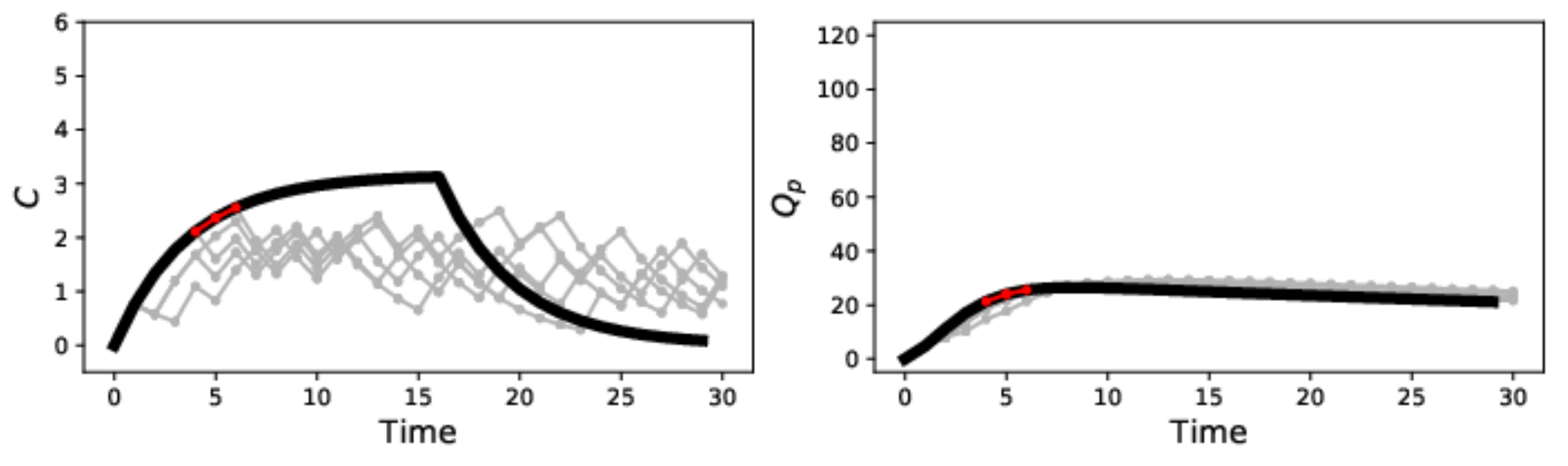}}
    \subfigure[Highlighted reliable transitions]{\includegraphics[width=0.4\textwidth]{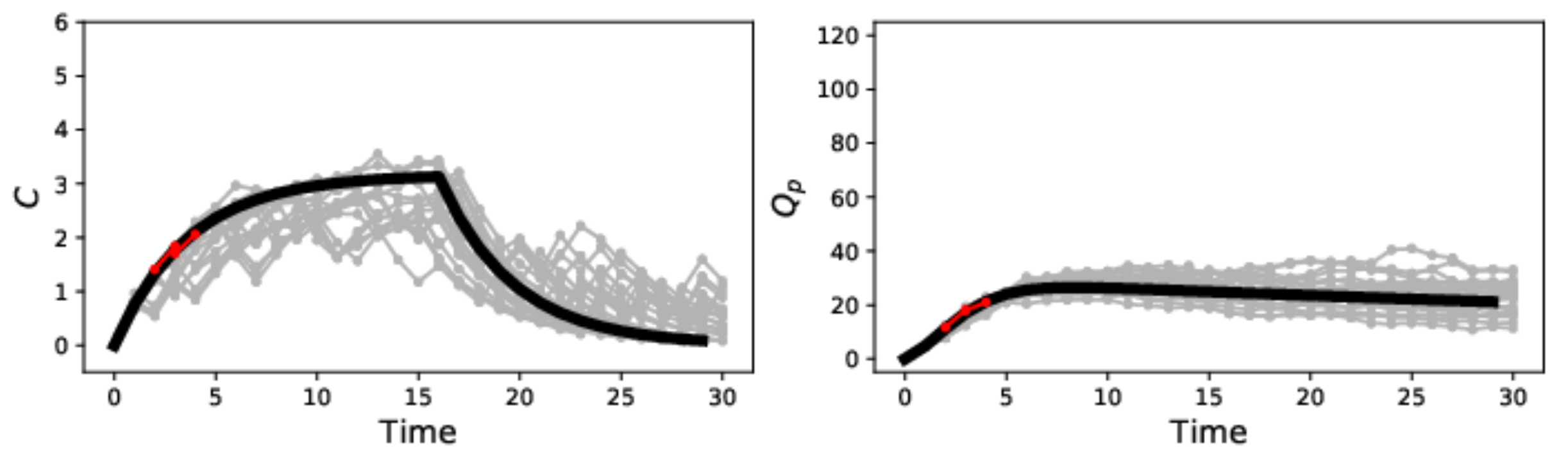}}
    \subfigure[Highlighted problematic transitions]{\includegraphics[width=0.4\textwidth]{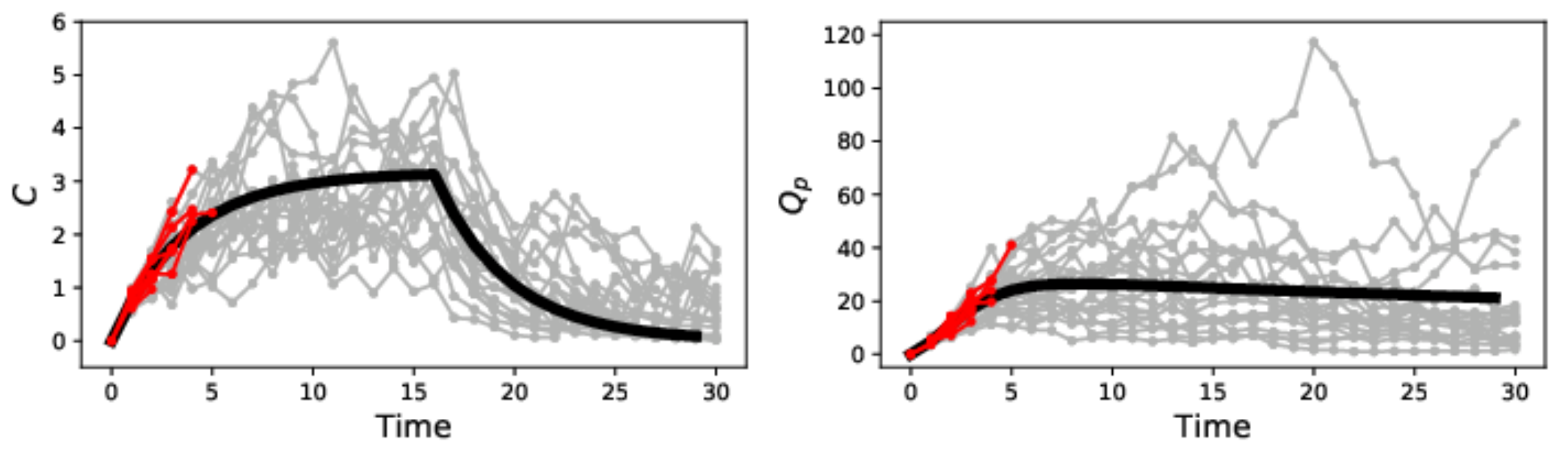}}
    \caption{\textbf{Influence analysis for simulated cancer data.} Analysis of synthetic cancer simulations demonstrates how influence analysis can differentiate between different diagnostics of the OPE process.}
    \label{fig:cancer}
\end{figure}

\begin{figure}[t]
    \centering
    \subfigure[Influence Distribution]{\includegraphics[width=0.35\textwidth]{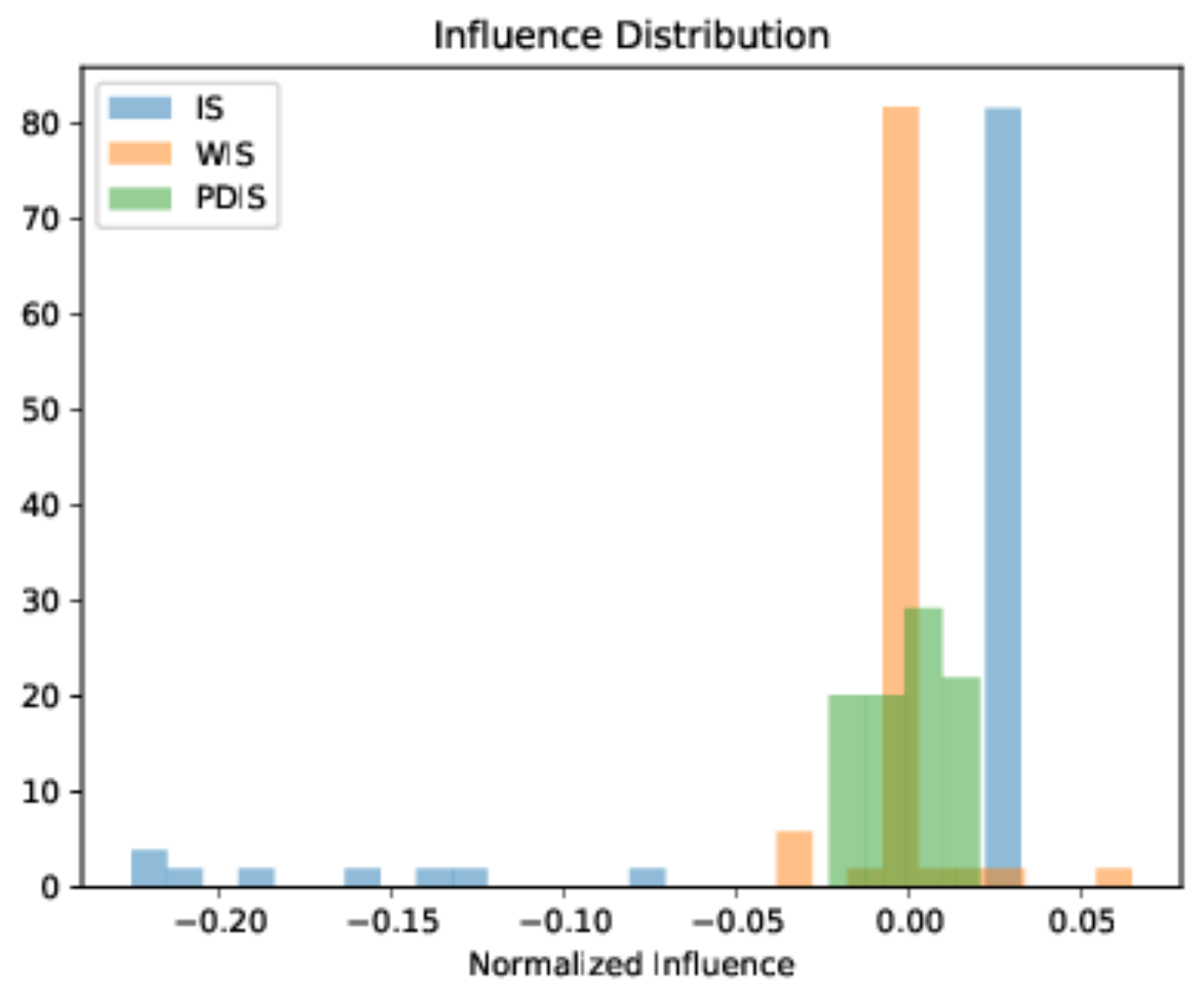}}
    \subfigure[IS]{\includegraphics[width=0.4\textwidth]{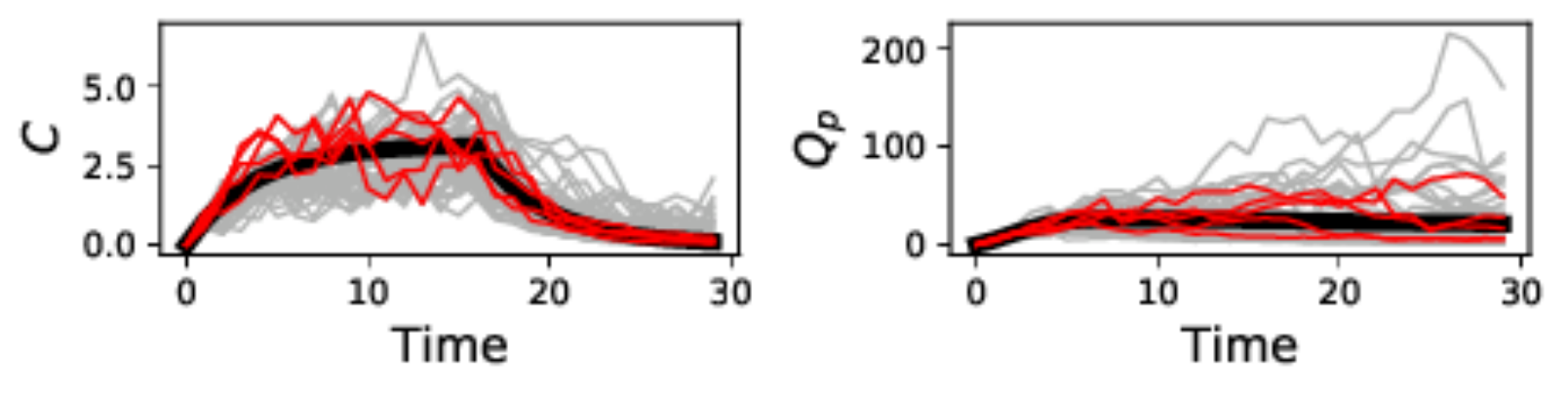}}
    \subfigure[WIS]{\includegraphics[width=0.4\textwidth]{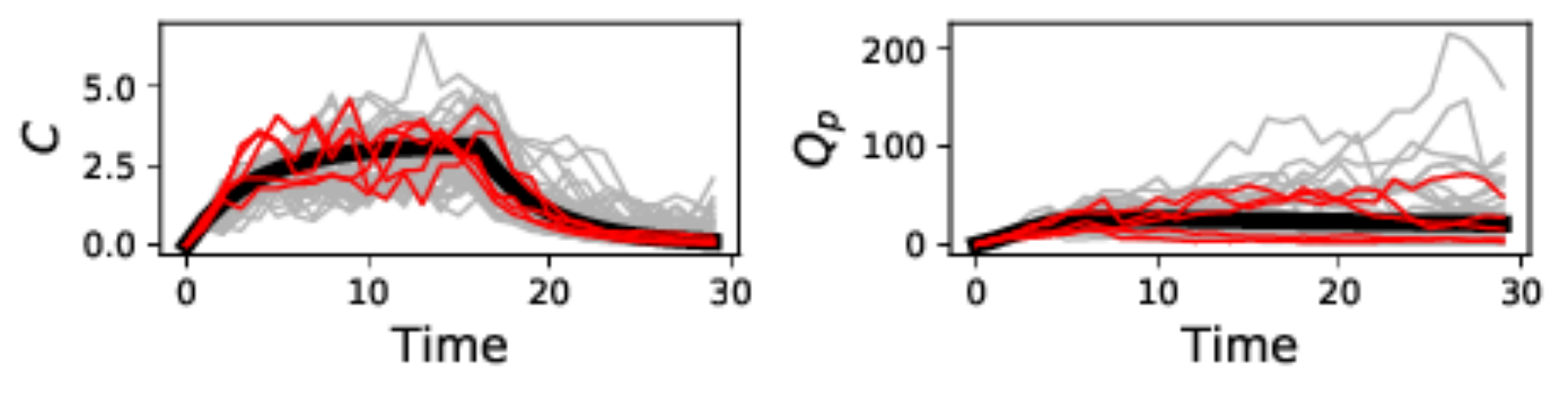}}
    \subfigure[PDIS]{\includegraphics[width=0.4\textwidth]{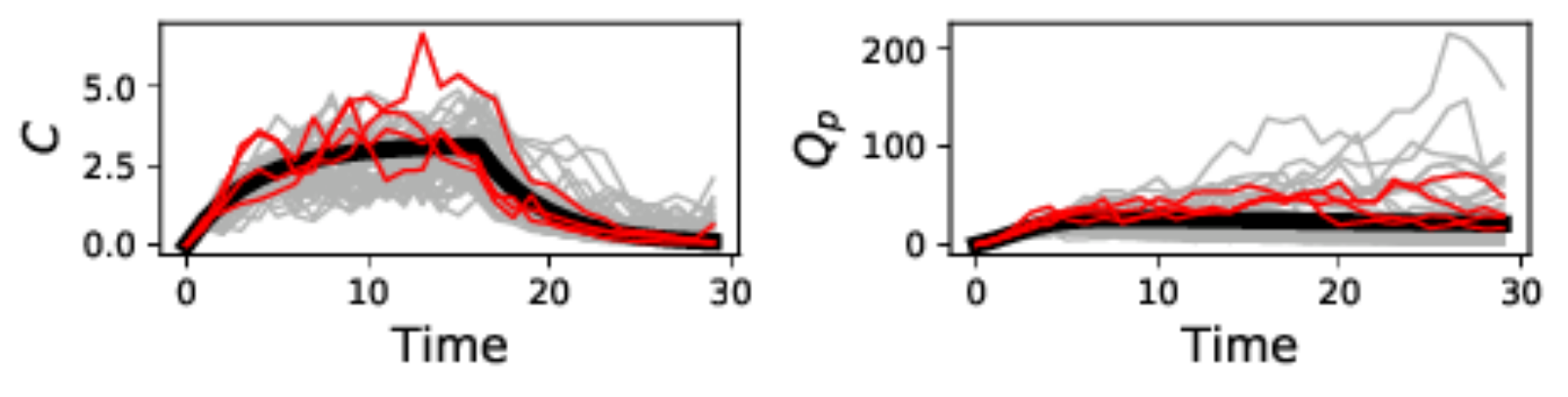}}
    \caption{\textbf{IS influence analysis for simulated cancer data.} For the same dataset, different estimators have different influence distributions, and for each estimator different trajectories have high influence.}
    \label{fig:is}
\end{figure}

To demonstrate the different ways in which influence analysis can allow domain experts to either increase our confidence in the validity of OPE or identify instances where they are invalid, we first present results on a simulator of cancer dynamics. The 4 dimensional states of the simulator approximate the dynamics of tumor growth, with actions consisting administration of chemotherapy at each timestep representing one month. See \citet{ribba2012tumor} for details.

In Figure \ref{fig:cancer} we present four cases in which we attempt to evaluate the policy of treating a patient for 15 months and then discontinuing chemotherapy until the end of treatment at 30 months. Each subplot in Figure \ref{fig:cancer} shows two of the four state variables as a function of time, under different conditions which might make evaluation more difficult, such as difference in behavior policy or stochasticity in the environment. The heavy black line represents the expectation of each state dimension at each time-step under the evaluation policy, while the grey lines represent observed transitions under the behavior policy which is $\epsilon$-greedy with respect to the evaluation policy. In all figures, we highlight in red all influential transitions our method would have highlighted for review by domain experts $(\tilde{I}_c = 0.05)$.

\paragraph{Case 1: OPE seems reliable.}
Figure \ref{fig:cancer}(a) represents a typical example where the OPE can easily be trusted. Despite the large difference between the evaluation and behavior policy $(\epsilon = 0.3)$, enough trajectories have been observed in the data to allow for proper evaluation, and no transition is flagged as being too influential. The value estimation error in this example is less than $1\%$ and our method correctly labels this dataset as reliable.

\paragraph{Case 2: Unevaluatable.}
Figure \ref{fig:cancer}(b) is similar in experimental conditions to (a) ($\epsilon = 0.3$ and deterministic transitions), but with less collected data, so that the observations needed to properly estimate the dynamics are not in the data. This can be seen by the lack of overlap between the observed transitions and the expected trajectory, and results in a $38\%$ value estimation error. In real life we will not know what the expected trajectory under the evaluation policy looks like, and therefore will not be able to make the comparison and detect the lack of overlap between transitions under the evaluation and behavior policies. However, our method highlights a very influential sequence which terminates at a dead-end, and thus will correctly flag this dataset as not sufficient for evaluation. Our method in this case is confident enough to dismiss the results of evaluation without need for domain experts, but can still inform experts on what type of data is lacking in order for evaluation to be feasible.

\paragraph{Case 3: Humans might help.}
In Figures \ref{fig:cancer}(c-d), $\epsilon = 0.3$, but the dynamics have different levels of stochasticity. The less stochastic dynamics in \ref{fig:cancer}(c) allow for relatively accurate evaluation ($8\%$ error) but our method identifies several influential transitions which must be presented to a domain expert. These transitions lie on the expected trajectory, and thus a clinician would verify that they represent a typical response of a patient to treatment. This is an example in which our method would allow a domain expert to verify the validity of the evaluation by examining the flagged influential transitions.

Conversely, in \ref{fig:cancer}(d) some extreme outliers lead to a large estimation error ($23\%$ error). The influential transitions identified by our method are exactly those which start close to the expected trajectory but deviate significantly from the expected dynamics. A domain expert presented with the these transitions would easily be able to note that the OPE heavily relies on atypical patients and rightly dismiss the validity of evaluation.

To summarize, we demonstrated that analysis of influences can both validate or invalidate the evaluation without need for domain experts, and in intermediate cases present domain experts with the correct queries required to gain confidence in the evaluation results or dismiss them.

\paragraph{Influence analysis for IS - Influence is a method specific quantity.}

In Figure \ref{fig:is} we present influence analysis results for the cancer environment, with different importance sampling methods. Unlike the FQE experiment where we performed influence analysis of the same estimator for different datasets, here we analyze the same dataset for three different OPE estimaors - IS, WIS and PDIS. In Figure \ref{fig:is} (a) we plot the distribution of the influence of all trajectories in the data, and see that the distributions are qualitatively different for each estimator. Furthermore, in \ref{fig:is} (b-d) we highlight the 5 most influential trajectories for each estimator, and see that they are different for each estimator. The key point we wish to highlight is that influence analysis identifies features of the interaction between a dataset and an estimator, and not of the data alone. This makes sense, as different OPE methods are robust or sensitive to different types of noise or artefacts in the data.

\subsection{Analysis of real ICU data - MIMIC III}

\begin{figure*}[t]
\centering
\includegraphics[width=0.43\textwidth]{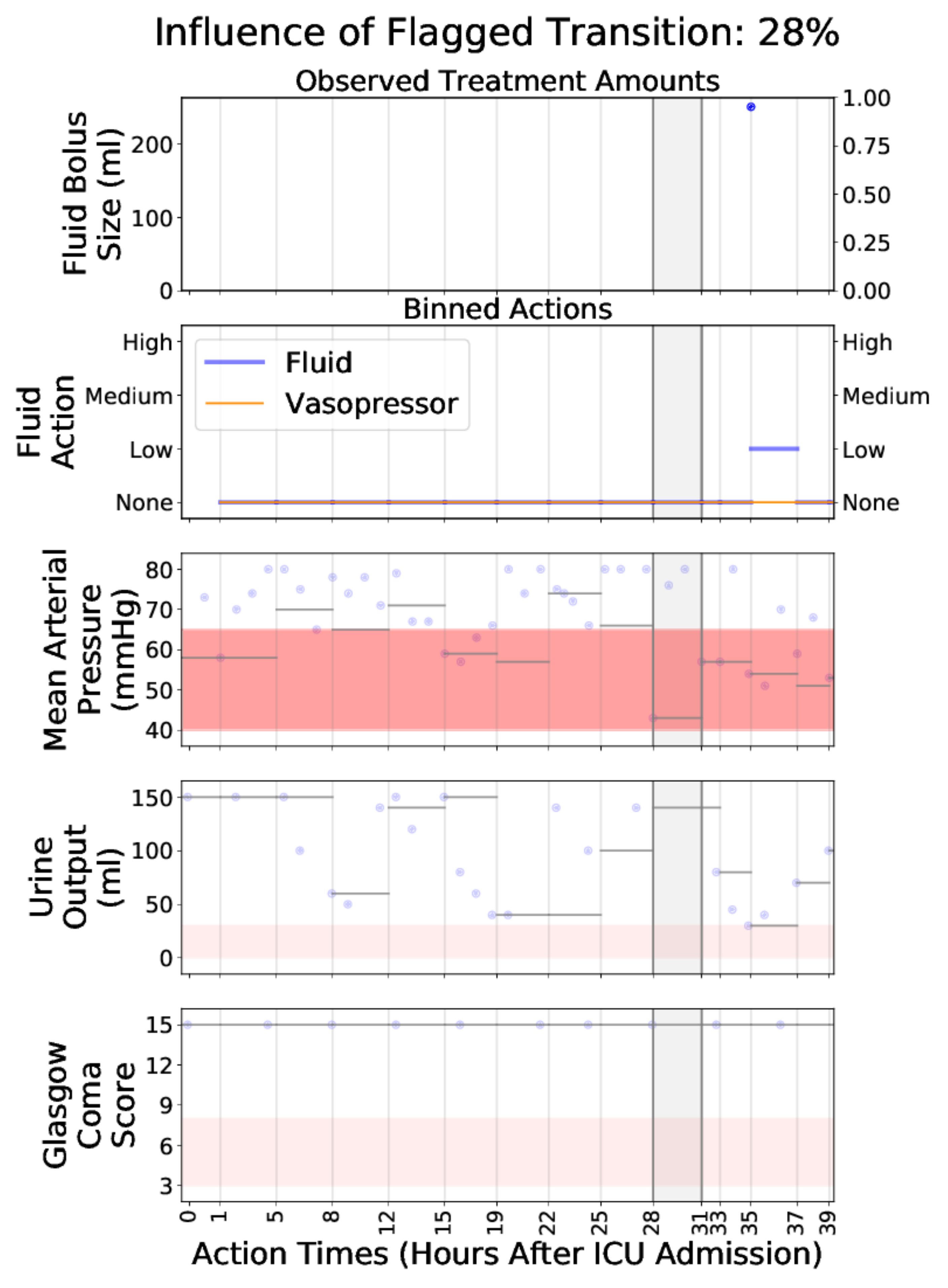}
\includegraphics[width=0.43\textwidth]{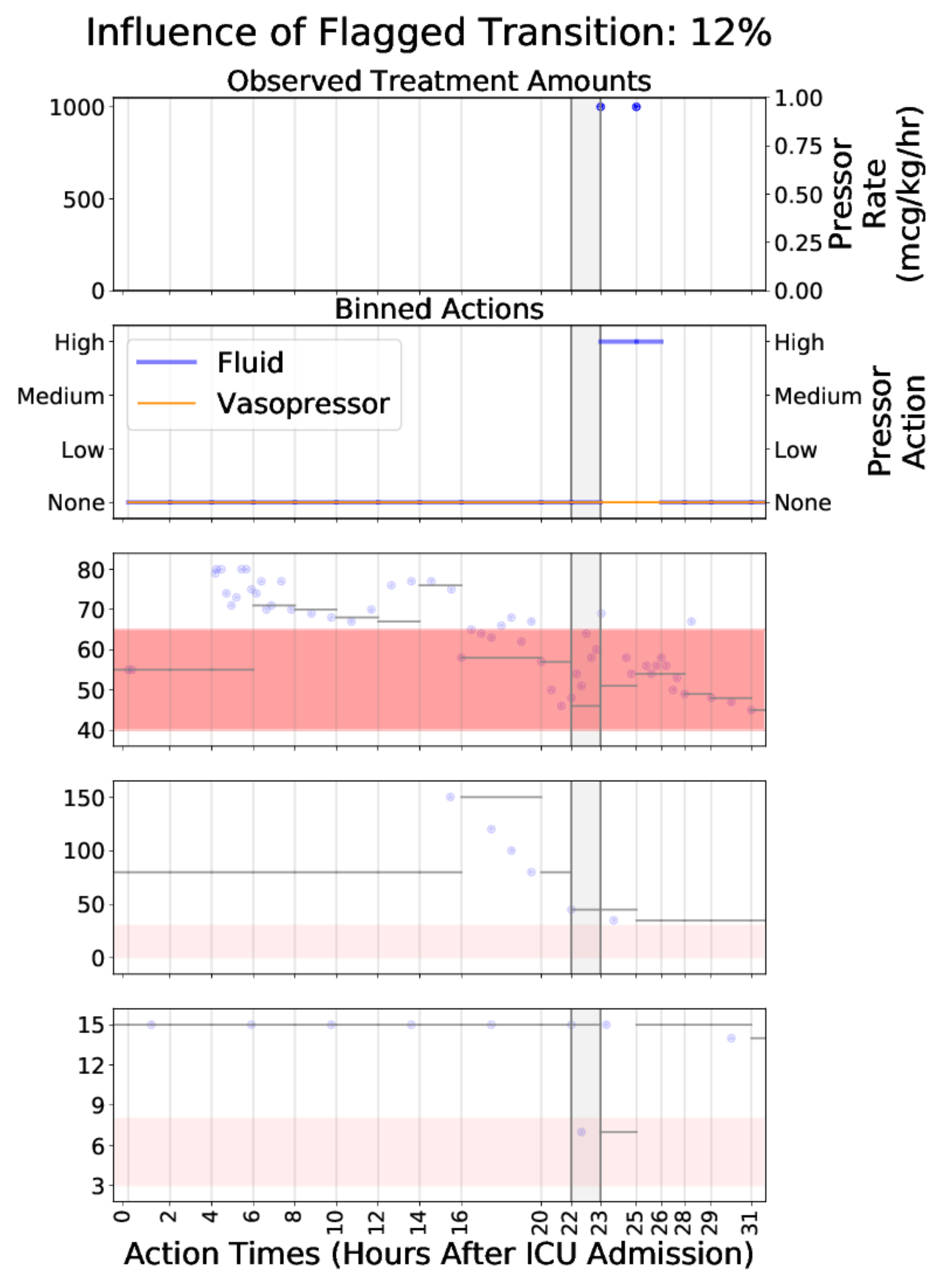}
\caption{Influence analysis on our real-world dataset discovered six transitions in the evaluation dataset that were especially influential on our OPE. We display two of them in this figure, see Appendix \ref{sec:additional-mimic-results} for the remaining four.} 
\label{fig:mimic}
\end{figure*}

To show how influence analysis can help debug OPE for a challenging healthcare task, we consider the management of acutely hypotensive patients in the ICU. Hypotension is associated with high morbidity and mortality \citep{jones2006emergency}, but management of these patients is not standardized as ICU patients are heterogeneous. Within critical care, there is scant high-quality evidence from randomized controlled trials to inform treatment guidelines \citep{de2018unexplained,girbes2019time}, which provides an opportunity for RL to help learn better treatment strategies. In collaboration with an intensivist, we use influence analysis to identify potential artefacts when performing OPE on a clinical dataset of acutely hypotensive patients. 

\paragraph{Data and evaluation policy.}
Our data source is a subset of the publicly available MIMIC-III dataset \citep{johnson2016mimic}. See Appendix \ref{appendix:mimic_details} for full details of the data preprocessing. Our final dataset consists of 346 patient trajectories (6777 transitions) for learning a policy and another 346 trajectories (6863 transitions) for evaluation of the policy via OPE and influence analysis. 

Our state space consists of 29 relevant clinical variables, summarizing current physiological condition and past actions. The two main treatments for hypotension are administration of an intravenous (IV) fluid bolus or initiation of vasopressors. We bin doses of each treatment into 4 categories for "none", "low", "medium" and "high", so that the full action space consists of 16 discrete actions. Each reward is a function of the next blood pressure (MAP) and takes values in $[-1,0]$. As an evaluation policy, we use the most common action of a state's 50 nearest neighbors. This is setup is equivalent to constructing a decision assistance tool for clinicians by recommending the common practice action for patients, and using OPE combined with influence analysis to estimate the efficacy of such a tool. See Appendix \ref{appendix:mimic_details} for more details on how we setup the RL problem formulation, and for the kernel function used to compute nearest-neighbors. 

\paragraph{Presenting queries to a practicing intensivist.}
Running influence analysis flags 6 influential $(\tilde{I}_C = 0.05)$. We show 2 of these transitions in Figure \ref{fig:mimic} and the rest in Appendix \ref{sec:additional-mimic-results}. While this analysis highlights individual transitions, our results figures display additional context before and after the suspect transition to help the clinician understand what might be going on.

In Figure \ref{fig:mimic}, each column shows a transition flagged by influence analysis. The top two rows show actions taken (actual treatments in the top row and binned actions in the second row). The remaining three rows show the most important state variables that inform the clinicians' decisions: blood pressure (MAP), urine output, and level of consciousness (GCS). For these three variables, the abnormal range is shaded in red, where the blood pressure shading is darker highlighting its direct relationship with the reward. Vertical grey lines represent timesteps, and the highlighted influential transition is shaded in grey.

\paragraph{Outcome: Identifying and removing an influential, buggy measurement.}
The two transitions in Figure \ref{fig:mimic} highlight potential problems in the dataset that have a large influence. In the first transition (left), a large drop in blood pressure is observed at the starting time of this transition, potentially indicating a dangerous hypotensive state. Suprisingly, the patient received no treatment, and this unusual transition has a 29\% influence on the OPE estimate. Given additional context just before and after the transition, showing otherwise stable MAP and GCS (patient was conscious and alert) as well as a normal urine output, the intensivist determined the single low MAP value was likely either a measurement error or a clinically insignificant transient episode of hypotension. After correcting the outlier MAP measurement to its most recent normal value (80mmHg) and then rerunning FQE and the influence analysis, the transition no longer has high influence and was not flagged.

\paragraph{Outcome: Identifying and correcting a temporal misalignment.}
The second highlighted transition (right) features a sudden drop in GCS and worsening MAP values, indicating a sudden deterioration of the patient's state, but treatment is not administered until the next timestep. The intensivist attributed this finding to a time stamp recording error. Again, influence analysis identified an inconsistency in the original data which had undue impact on evaluation. After correcting the inconsistency by shifting the two fluid treatments back by one timestep each, we found that the transition no longer had high influence and was not flagged.

\section{Discussion}

A key aim of this paper is to formulate a framework for using domain expertise to help in evaluating the trustworthiness of OPE methods for noisy and confounded observational data. The motivation for this research direction is the intersection of two realities: for messy real-world applications, the data itself might never be enough; and domain experts will always need to be involved in the integration of decision support tools, so we should incorporate their expertise into the evaluation process. We showcased influence analysis as one way of performing this task for value-based and IS OPE, but emphasize that such measures can and should be incorporated into other methods as well. For example, when modeling the dynamics in model-based OPE, the results can be tested for their agreement with expert intuition.

We stress that research to integrate human input into OPE methods to increase their reliability complements, and does not replace, the approaches for estimating error bounds and uncertainties over the errors of OPE estimates. The fact that traditional theoretical error bounds rely so heavily on assumptions which are generally impossible to verify from the data alone highlights the need for other techniques for gauging to what extent these assumptions hold.

\bibliography{references}

\begin{thebibliography}{27}
\providecommand{\natexlab}[1]{#1}
\providecommand{\url}[1]{\texttt{#1}}
\expandafter\ifx\csname urlstyle\endcsname\relax
  \providecommand{\doi}[1]{doi: #1}\else
  \providecommand{\doi}{doi: \begingroup \urlstyle{rm}\Url}\fi

\bibitem[Agniel et~al.(2018)Agniel, Kohane, and Weber]{agniel2018biases}
Agniel, D., Kohane, I.~S., and Weber, G.~M.
\newblock Biases in electronic health record data due to processes within the
  healthcare system: retrospective observational study.
\newblock \emph{Bmj}, 361:\penalty0 k1479, 2018.

\bibitem[Asfar et~al.(2014)Asfar, Meziani, Hamel, Grelon, Megarbane, Anguel,
  Mira, Dequin, Gergaud, Weiss, et~al.]{asfar2014high}
Asfar, P., Meziani, F., Hamel, J.-F., Grelon, F., Megarbane, B., Anguel, N.,
  Mira, J.-P., Dequin, P.-F., Gergaud, S., Weiss, N., et~al.
\newblock High versus low blood-pressure target in patients with septic shock.
\newblock \emph{N Engl J Med}, 370:\penalty0 1583--1593, 2014.

\bibitem[Cook \& Weisberg(1980)Cook and Weisberg]{cook1980characterizations}
Cook, R.~D. and Weisberg, S.
\newblock Characterizations of an empirical influence function for detecting
  influential cases in regression.
\newblock \emph{Technometrics}, 22\penalty0 (4):\penalty0 495--508, 1980.

\bibitem[Dann et~al.(2018)Dann, Li, Wei, and Brunskill]{dann2018policy}
Dann, C., Li, L., Wei, W., and Brunskill, E.
\newblock Policy certificates: Towards accountable reinforcement learning.
\newblock \emph{arXiv preprint arXiv:1811.03056}, 2018.

\bibitem[de~Grooth et~al.(2018)de~Grooth, Postema, Loer, Parienti, Oudemans-van
  Straaten, and Girbes]{de2018unexplained}
de~Grooth, H.-J., Postema, J., Loer, S.~A., Parienti, J.-J., Oudemans-van
  Straaten, H.~M., and Girbes, A.~R.
\newblock Unexplained mortality differences between septic shock trials: a
  systematic analysis of population characteristics and control-group mortality
  rates.
\newblock \emph{Intensive care medicine}, 44\penalty0 (3):\penalty0 311--322,
  2018.

\bibitem[Ernst et~al.(2005)Ernst, Geurts, and Wehenkel]{ernst2005tree}
Ernst, D., Geurts, P., and Wehenkel, L.
\newblock Tree-based batch mode reinforcement learning.
\newblock \emph{Journal of Machine Learning Research}, 6\penalty0
  (Apr):\penalty0 503--556, 2005.

\bibitem[Girbes \& de~Grooth(2019)Girbes and de~Grooth]{girbes2019time}
Girbes, A. R.~J. and de~Grooth, H.-J.
\newblock Time to stop randomized and large pragmatic trials for intensive care
  medicine syndromes: the case of sepsis and acute respiratory distress
  syndrome.
\newblock \emph{Journal of Thoracic Disease}, 12\penalty0 (S1), 2019.
\newblock ISSN 2077-6624.
\newblock URL \url{http://jtd.amegroups.com/article/view/33636}.

\bibitem[Gottesman et~al.(2018)Gottesman, Johansson, Meier, Dent, Lee,
  Srinivasan, Zhang, Ding, Wihl, Peng, et~al.]{gottesman2018evaluating}
Gottesman, O., Johansson, F., Meier, J., Dent, J., Lee, D., Srinivasan, S.,
  Zhang, L., Ding, Y., Wihl, D., Peng, X., et~al.
\newblock Evaluating reinforcement learning algorithms in observational health
  settings.
\newblock \emph{arXiv preprint arXiv:1805.12298}, 2018.

\bibitem[Gottesman et~al.(2019{\natexlab{a}})Gottesman, Johansson, Komorowski,
  Faisal, Sontag, Doshi-Velez, and Celi]{gottesman2019guidelines}
Gottesman, O., Johansson, F., Komorowski, M., Faisal, A., Sontag, D.,
  Doshi-Velez, F., and Celi, L.~A.
\newblock Guidelines for reinforcement learning in healthcare.
\newblock \emph{Nat Med}, 25\penalty0 (1):\penalty0 16--18, 2019{\natexlab{a}}.

\bibitem[Gottesman et~al.(2019{\natexlab{b}})Gottesman, Liu, Sussex, Brunskill,
  and Doshi-Velez]{gottesman2019combining}
Gottesman, O., Liu, Y., Sussex, S., Brunskill, E., and Doshi-Velez, F.
\newblock Combining parametric and nonparametric models for off-policy
  evaluation.
\newblock In \emph{International Conference on Machine Learning}, pp.\
  2366--2375, 2019{\natexlab{b}}.

\bibitem[Hanna et~al.(2017)Hanna, Stone, and Niekum]{hanna2017bootstrapping}
Hanna, J.~P., Stone, P., and Niekum, S.
\newblock Bootstrapping with models: Confidence intervals for off-policy
  evaluation.
\newblock In \emph{Thirty-First AAAI Conference on Artificial Intelligence},
  2017.

\bibitem[Jiang \& Li(2015)Jiang and Li]{jiang2015doubly}
Jiang, N. and Li, L.
\newblock Doubly robust off-policy value evaluation for reinforcement learning.
\newblock \emph{arXiv preprint arXiv:1511.03722}, 2015.

\bibitem[Johnson et~al.(2016)Johnson, Pollard, Shen, Li-wei, Feng, Ghassemi,
  Moody, Szolovits, Celi, and Mark]{johnson2016mimic}
Johnson, A.~E., Pollard, T.~J., Shen, L., Li-wei, H.~L., Feng, M., Ghassemi,
  M., Moody, B., Szolovits, P., Celi, L.~A., and Mark, R.~G.
\newblock Mimic-iii, a freely accessible critical care database.
\newblock \emph{Scientific data}, 3:\penalty0 160035, 2016.

\bibitem[Jones et~al.(2006)Jones, Yiannibas, Johnson, and
  Kline]{jones2006emergency}
Jones, A.~E., Yiannibas, V., Johnson, C., and Kline, J.~A.
\newblock Emergency department hypotension predicts sudden unexpected
  in-hospital mortality: a prospective cohort study.
\newblock \emph{Chest}, 130\penalty0 (4):\penalty0 941--946, 2006.

\bibitem[Koh \& Liang(2017)Koh and Liang]{koh2017understanding}
Koh, P.~W. and Liang, P.
\newblock Understanding black-box predictions via influence functions.
\newblock In \emph{Proceedings of the 34th International Conference on Machine
  Learning-Volume 70}, pp.\  1885--1894. JMLR. org, 2017.

\bibitem[Komorowski et~al.(2018)Komorowski, Celi, Badawi, Gordon, and
  Faisal]{komorowski2018artificial}
Komorowski, M., Celi, L.~A., Badawi, O., Gordon, A.~C., and Faisal, A.~A.
\newblock The artificial intelligence clinician learns optimal treatment
  strategies for sepsis in intensive care.
\newblock \emph{Nature medicine}, 24\penalty0 (11):\penalty0 1716--1720, 2018.

\bibitem[Lage et~al.(2018)Lage, Ross, Gershman, Kim, and
  Doshi-Velez]{lage2018human}
Lage, I., Ross, A., Gershman, S.~J., Kim, B., and Doshi-Velez, F.
\newblock Human-in-the-loop interpretability prior.
\newblock In \emph{Advances in Neural Information Processing Systems}, pp.\
  10159--10168, 2018.

\bibitem[Le et~al.(2019)Le, Voloshin, and Yue]{le2019batch}
Le, H.~M., Voloshin, C., and Yue, Y.
\newblock Batch policy learning under constraints.
\newblock \emph{arXiv preprint arXiv:1903.08738}, 2019.

\bibitem[Munos \& Moore(2002)Munos and Moore]{munos2002variable}
Munos, R. and Moore, A.
\newblock Variable resolution discretization in optimal control.
\newblock \emph{Machine learning}, 49\penalty0 (2-3):\penalty0 291--323, 2002.

\bibitem[Oberst \& Sontag(2019)Oberst and Sontag]{oberst2019counterfactual}
Oberst, M. and Sontag, D.
\newblock Counterfactual off-policy evaluation with gumbel-max structural
  causal models.
\newblock In \emph{International Conference on Machine Learning}, pp.\
  4881--4890, 2019.

\bibitem[Precup(2000)]{precup2000eligibility}
Precup, D.
\newblock Eligibility traces for off-policy policy evaluation.
\newblock \emph{Computer Science Department Faculty Publication Series}, pp.\
  ~80, 2000.

\bibitem[Ribba et~al.(2012)Ribba, Kaloshi, Peyre, Ricard, Calvez, Tod,
  {\v{C}}ajavec-Bernard, Idbaih, Psimaras, Dainese, et~al.]{ribba2012tumor}
Ribba, B., Kaloshi, G., Peyre, M., Ricard, D., Calvez, V., Tod, M.,
  {\v{C}}ajavec-Bernard, B., Idbaih, A., Psimaras, D., Dainese, L., et~al.
\newblock A tumor growth inhibition model for low-grade glioma treated with
  chemotherapy or radiotherapy.
\newblock \emph{Clinical Cancer Research}, 18\penalty0 (18):\penalty0
  5071--5080, 2012.

\bibitem[Sutton \& Barto(2018)Sutton and Barto]{sutton2018reinforcement}
Sutton, R.~S. and Barto, A.~G.
\newblock \emph{Reinforcement learning: An introduction}.
\newblock MIT press, 2018.

\bibitem[Tamuz et~al.(2011)Tamuz, Liu, Belongie, Shamir, and
  Kalai]{tamuz2011adaptively}
Tamuz, O., Liu, C., Belongie, S., Shamir, O., and Kalai, A.~T.
\newblock Adaptively learning the crowd kernel.
\newblock \emph{arXiv preprint arXiv:1105.1033}, 2011.

\bibitem[Thomas \& Brunskill(2016)Thomas and Brunskill]{thomas2016data}
Thomas, P. and Brunskill, E.
\newblock Data-efficient off-policy policy evaluation for reinforcement
  learning.
\newblock In \emph{International Conference on Machine Learning}, pp.\
  2139--2148, 2016.

\bibitem[Thomas et~al.(2015)Thomas, Theocharous, and
  Ghavamzadeh]{thomas2015high}
Thomas, P.~S., Theocharous, G., and Ghavamzadeh, M.
\newblock High-confidence off-policy evaluation.
\newblock In \emph{Twenty-Ninth AAAI Conference on Artificial Intelligence},
  2015.

\bibitem[Thomas et~al.(2019)Thomas, da~Silva, Barto, Giguere, Brun, and
  Brunskill]{thomas2019}
Thomas, P.~S., da~Silva, B.~C., Barto, A.~G., Giguere, S., Brun, Y., and
  Brunskill, E.
\newblock Preventing undesirable behavior of intelligent machines.
\newblock \emph{Science}, 366\penalty0 (6468):\penalty0 999--1004, 2019.

\end{thebibliography}
\bibliographystyle{icml2020}

\clearpage
\newpage

\appendix

\section{Derivations for Kernel-Based FQE}
\label{appendix:fqe}

\subsection{Proof of Proposition \ref{prop:matrix_fqe}}
\label{appendix:proof_matrix_fqe}

\setcounter{proposition}{0}
\begin{proposition}
\label{prop:matrix_fqe_restated}
For all transitions in the dataset, the values for corresponding state-action pairs are given by
\begin{align}
    \mathbf{\hat{q}}'_t &= \left( \sum_{t'=1}^t \gamma^{t'-1} \mathbf{M}'^{t'} \right) \mathbf{r} \equiv \mathbf{\Phi}'_t \mathbf{r} \label{eq:q_prime_estimate} \\
    \mathbf{\hat{q}}_t &= \mathbf{M} \left( \sum_{t'=1}^t \left( \gamma \mathbf{M}' \right)^{t'-1} \right) \mathbf{r} \equiv \mathbf{\Phi}_t \mathbf{r}.
\end{align}
where $\hat{q}'_{t, i}$ and $\hat{q}_{t, i}$ are the estimated policy values at $(x'^{(i)}, \pi_e(x'^{(i)}))$ and $(x^{(i)}, a^{(i)})$, respectively, for the observed transition $i$
\end{proposition}

\begin{proof}

We first prove \ref{eq:q_prime_estimate} by induction. We start by noting that for a given observed transition, $i$, averaging over all observations $j$ such that $\Delta_{ij}$ holds can be written as $\frac{1}{N_i}\sum_{j:\Delta_{ij}} = \sum_{j} \frac{\mathbb{I}({\Delta_{ij}})}{N_i} = \sum_{j} M_{ij}$. Similarly, averaging over all $j$ such that $\Delta_{i'j}$ holds can be written as $\sum_{j} M'_{ij}$. Therefore, if $u(x, a)$ is some function over the sttate-action space and $\mathbf{u}$ is a vector containing the quantity $u_i = u(x^{(i)}, a^{(i)})$ for every $(x^{(i)}, a^{(i)})$, then the nearest-neighbors estimation of $u(x'^{(i)}, \pi(x'^{(i)}))$ is given by $\lbrack \mathbf{M}'\mathbf{u} \rbrack_i$.

Given the formulation above, for $t=1$, $\hat{q}'_{t,i}$ estimates the reward at $(x'^{(i)}, \pi(x'^{(i)}))$, and can be written as:

\begin{align}
    \mathbf{\hat{q}'}_1 = \mathbf{M}' \mathbf{r}.
\end{align}

For $t > 1$, assume $\mathbf{\hat{q}}'_{t-1} = \left( \sum_{t'=1}^{t-1} \gamma^{t'-1} \mathbf{M}'^{t'} \right) \mathbf{r}$. Then

\begin{align}
    \mathbf{\hat{q}}'_t &= \mathbf{M}' \left( \mathbf{r} + \gamma \mathbf{\hat{q}}'_{t-1} \right) \\ \nonumber
    &= \mathbf{M}' \left[ \mathbf{r} + \gamma \left( \sum_{t'=1}^{t-1} \gamma^{t'-1} \mathbf{M}'^{t'} \right) \mathbf{r} \right] \\ \nonumber
    &= \left( \mathbf{M}' + \sum_{t'=1}^{t-1} \gamma^{t'} \mathbf{M}'^{t'+1} \right) \mathbf{r} \\ \nonumber
    &= \left(\sum_{t'=0}^{t-1} \gamma^{t'} \mathbf{M}'^{t'+1} \right) \mathbf{r} \\ \nonumber
    &= \left( \sum_{t'=1}^{t} \gamma^{t'-1} \mathbf{M}'^{t'} \right) \mathbf{r} \equiv \mathbf{\Phi}'_t \mathbf{r},
\end{align}

completing the proof of \ref{eq:q_prime_estimate}. To estimate $\mathbf{\hat{q}}_t$, we write $\hat{q}_{t,i} = \frac{1}{N_i}\sum_{j:\Delta_{ij}} \left( r^{(j)} + \gamma \hat{q}'_{t-1,j} \right)$ or in matrix notation.

\begin{align}
    \mathbf{\hat{q}}_t &= \mathbf{M} \left( \mathbf{r} + \gamma \mathbf{\hat{q}}'_{t-1} \right)\\ \nonumber
    &= \mathbf{M} \left( \mathbf{I} + \gamma \left( \sum_{t'=1}^{t-1} \gamma^{t'-1} \mathbf{M}'^{t'} \right) \right) \mathbf{r} \\ \nonumber
    &= \mathbf{M} \left( \sum_{t'=1}^{t} \left( \gamma \mathbf{M}' \right)^{t'-1} \right) \mathbf{r} \equiv \mathbf{\Phi}_t \mathbf{r}. 
\end{align}

\end{proof}

\subsection{Computational complexity.} 
\label{appendix:kernel_fqe_complexity}

Computation of a single  influence value, $\tilde{I}_{i, j}$ requires summation over all transitions $k$ that satisfy $\Delta_{k'j}$. Denote the number of such neighbors by $N_{j'}^*$ \footnote{Note that $N_{j'}^*$, which counts all $k$ that satisfy $\Delta_{k'j}$, is subtly different from the quantity $N_{j'}$ introduced in section \ref{sec:kernel_based_fqe}, which counts all $k$ that satisfy $\Delta_{j'k}$.}. We expect $N_{j'}^*$ to be small and not scale with the size of the dataset, and also $\tilde{I}_{i, j}$ is inversely proportional to $N_{j'}^*$. Thus, if we only compute the influence of transitions such that $N_{j'}^* < \frac{v_{\text{max}}}{\hat{v}\tilde{I}_c} \equiv N_{j',c}^*$, where $v_{\text{max}}$ is the maximum possible value, we are guaranteed not to miss any transitions with influence larger than our threshold $\tilde{I}_c$. Since $N_{j',c}^*$ does not scale with the size of the data, computation of a single individual influence can effectively be done in constant time. Performing influence analysis on a full dataset requires computing the influences of all transitions on all initial transitions, and therefore takes $\mathcal{O}(N |\mathcal{D}^*_0|)$ time.

In our matrix formulation, the FQE evaluation itself is bottlenecked by computing the matrix $\Phi$, which includes computation of powers of $M'$. Because $M'$ is a sparse matrix (each row $i$ only has $N_i$ nonzero elements), the matrix multiplication itself can be done in $\mathcal{O}(N^2)$ rather than $\mathcal{O}(N^3)$ time, and the entire evaluation is done in $\mathcal{O}(N^2 T)$ time. Importantly, the influence analysis analyzing all transitions has lower complexity than the OPE, and should not significantly increase the computational cost of the evaluation pipeline.

\section{Derivations Linear Least-Squares FQE}
\label{appendix:ls_fqe}
\begin{proposition}
\label{prop:ls_fqe}
The the linear least square solution of fitted Q evaluation is 
$
    (\Feature^\top \Feature - \gamma \Feature^\top \Feature_p)^{-1}\Feature^\top \mathbf{r}
$
\end{proposition}

\begin{proof}
The least-square solution of parameter vector $\mathbf{w}$ can be found by minimizing the following square error of the Bellman equation for all $(x,a)$ in the dataset:
\begin{align}
    &\left(\hat{q}(x,a) - r(x,a) - \gamma \hat{q}(x', \pi_e(x')) \right)^2 
\end{align}
Plugging in $\hat{q}(x,a) = \feature(x,a)^\top \mathbf{w}$, the square error is
\begin{align}
    \left( \feature(x,a)^\top \mathbf{w} - r(x,a) - \gamma \feature_{\pi}(x')^\top \mathbf{w}\right)^2
\end{align}
By definition of $\Feature$ and $\Feature_p$, the mean square error over the $N$ samples is:
\begin{align}
    \left\| \Feature\mathbf{w} - \mathbf{r} - \gamma \Feature_p \mathbf{w}\right\|_2^2
\end{align}
The least square solution is:
\begin{align}
    \mathbf{w} = (\Feature^\top \Feature)^{-1}\Feature^\top (\mathbf{r} + \gamma \Feature_p \mathbf{w}) \\
    (\Feature^\top \Feature) \mathbf{w} =\Feature^\top (\mathbf{r} + \gamma \Feature_p \mathbf{w})\\
    (\Feature^\top \Feature - \gamma \Feature^\top \Feature_p) \mathbf{w} =\Feature^\top \mathbf{r} \\
    \mathbf{w} = (\Feature^\top \Feature - \gamma \Feature^\top \Feature_p)^{-1}\Feature^\top \mathbf{r}
\end{align}

\end{proof}

\begin{proposition}
\label{prop:ls_fqe_influence}
Let $\mathbf{C}_{-j} = (\Feature_{-j}^\top \Feature_{-j} - \gamma \Feature_{-j}^\top \Feature_{p,-j})$ and $\mathbf{C} = (\Feature^\top \Feature - \gamma \Feature^\top \Feature_p)$.
\begin{align*}
    \mathbf{w}_{-j} &= \left(\mathbf{C}_{-j}\right)^{-1}\left( \Feature^\top \mathbf{r} - r^{(j)} \feature_j \right)\\
    \left(\mathbf{C}_{-j}\right)^{-1} &= \mathbf{B}_{j} - \frac{\gamma \mathbf{B}_{j} \feature_j \feature_{\pi,j}^\top \mathbf{B}_{j}}{1+ \gamma \feature_{\pi,j}^{\top}\mathbf{B}_{j}\feature_{j}} \\
\end{align*}
where
\begin{align*}
    \mathbf{B}_{j} &= \mathbf{C}^{-1} + \frac{\mathbf{C}^{-1} \feature_j \feature_j^\top \mathbf{C}^{-1}}{1- \feature_{j}^{\top}\mathbf{C}^{-1}\feature_{j}}
\end{align*}
\end{proposition}
\begin{proof}
By the list squares solution of FQE,
$\mathbf{w}_{-j}$ equals $(\Feature_{-j}^\top \Feature_{-j} - \gamma \Feature_{-j}^\top \Feature_{p,-j})^{-1}\Feature_{-j}^\top \mathbf{r}_{-i}$.
Since 
$\Feature_{-j}^\top \mathbf{r}_{-j} = \Feature^\top \mathbf{r} - r^{(i)} \feature_j,
$
we have that $\mathbf{w}_{-j} = \left(\mathbf{C}_{-j}\right)^{-1}\left( \Feature^\top \mathbf{r} - r^{(j)} \feature_j \right)$. Then
\begin{align}
    \mathbf{C}_{-j} &= \mathbf{C} - \feature_j \feature_j^\top + \gamma \feature_j \feature_{\pi,j}^\top,
\end{align}
because
\begin{align}
    \Feature_{-j}^\top \Feature_{-j} &= \Feature^\top \Feature - \feature_j \feature_j^\top\\
    \Feature_{-j}^\top \Feature_{p,-j} &= \Feature^\top \Feature_p - \feature_j \feature_{\pi,j}^\top
\end{align}
This indicate $\mathbf{C}_{-j}$ equals $\mathbf{C}$ plus two rank-1 matrices. Fortunately, we can store $\mathbf{C}^{-1}$ when we compute $\mathbf{w}$ and $\hat{q}$. The following result named Sherman–Morrison formula allow us to compute $\mathbf{C}_{-j}^{-1}$ from $\mathbf{C}^{-1}$ in an efficient way. For any invertible matrix $\mathbf{A} \in \mathbb{R}^{d \times d}$ and vector $u$, $v \in \mathbb{R}^d$:
\begin{align}
    (\mathbf{A} + uv^\top)^{-1} = \mathbf{A}^{-1} - \frac{\mathbf{A}^{-1} uv^\top \mathbf{A}^{-1}}{1+ v^{\top}\mathbf{A}^{-1}u}
\end{align}
Then if we define $\mathbf{B}_{j} \equiv \left(\mathbf{C} - \feature_j \feature_j^\top\right)^{-1}$, we have that
\begin{align}
    \mathbf{B}_{j} &= \mathbf{C}^{-1} + \frac{\mathbf{C}^{-1} \feature_j \feature_j^\top \mathbf{C}^{-1}}{1- \feature_{j}^{\top}\mathbf{C}^{-1}\feature_{j}}\\
    \left(\mathbf{C}_{-j}\right)^{-1} &= \mathbf{B}_{j} - \frac{\gamma \mathbf{B}_{j} \feature_j \feature_{\pi,j}^\top \mathbf{B}_{j}}{1+ \gamma \feature_{\pi,j}^{\top}\mathbf{B}_{j}\feature_{j}} 
\end{align}
\end{proof}

\section{Influence computation of importance sampling methods}
\label{appendix:is_computation}

The standard importance sampling (IS) estimator is given by

\begin{equation}
\label{eq:is}
    \hat{v}^{\pi_e}_{IS} = \frac{1}{N}\sum_{n=1}^N w_{0:T}^{(n)} g_T^{(n)}, 
\end{equation}

where the summation is over all $N$ trajectories in the dataset, and the importance sampling weight $w_{0:t}$ is given by

\begin{equation}
    \label{eq:is_weight}
    w_{0:t}^{(n)} = \prod_{t'=0}^t \frac{\pi_e(a_{t'}^{(n)}|s_{t'}^{(n)})}{\pi_b(a_{t'}^{(n)}|s_{t'}^{(n)})}.
\end{equation}

The total influence of trajectory $j$ is then

\begin{align}
    \label{eq:is_inf}
    I_j &= \hat{v}_{-j} - \hat{v} \\ \nonumber
    &= \frac{1}{N-1}\sum_{n \neq j} w_{0:T}^{(n)} g_T^{(n)} - \frac{1}{N}\sum_{n=1}^{N} w_{0:T}^{(n)} g_T^{(n)} \\ \nonumber
    &= \sum_{n=1}^{N} w_{0:T}^{(n)} g_T^{(n)} \left( \frac{1}{N-1} - \frac{1}{N} \right) - \frac{1}{N-1} w_{0:T}^{(j)} g_T^{(j)} \\ \nonumber
    &= \frac{1}{N(N-1)} \sum_{n=1}^{N} w_{0:T}^{(n)} g_T^{(n)} - \frac{1}{N-1} w_{0:T}^{(j)} g_T^{(j)} \\ \nonumber
    &= \frac{1}{N-1} \left( \hat{v}_{IS} - w_{0:T}^{(j)} g_T^{(j)} \right).
\end{align}

This relation is nothing more then the fact that removing the $j^{th}$ sample from an average over $N$ samples, $\bar{x} = \frac{1}{N} \sum x^{(n)}$, changes the average by $\frac{1}{N-1}(\bar{x} - x^{(j)})$.

Using the same derivation we can compute the influence of the per-decision importance sampling estimator (PDIS) and doubly-robust importance sampling estimator (DR):

\paragraph{PDIS}

For the PDIS estimator, given by

\begin{align}
    \hat{v}_{PDIS}^{\pi_e} = \frac{1}{N} \sum_{n=1}^N \sum_{t=0}^{T-1} w_{0:t}^{(n)} \gamma^t r_t^{(n)},
\end{align}

the total influence of trajectory $j$ is given by

\begin{align}
    I_j = \frac{1}{N-1} \left( \hat{v}_{PDIS} - \sum_{t=0}^{T-1} w_{0:t}^{(j)} \gamma^t r_t^{(j)} \right).
\end{align}

\paragraph{DR}

For the DR estimator \citep{jiang2015doubly}, given by

\begin{align}
    \hat{v}_{DR}^{\pi_e} &= \frac{1}{N} \sum_{n=1}^N \Bigg( \Bigg. \sum_{t=0}^{T-1} \gamma^t \Bigg( \Bigg. w_{0:t}^{(n)} r_t^{(n)} \\ \nonumber
    &- w_{0:t}^{(n)} \tilde{q}(x^{(n)}_t, a^{(n)}_t) + w_{0:t-1}^{(n)} \tilde{v}(x^{(n)}_t) \Bigg. \Bigg) \Bigg. \Bigg),
\end{align}

where $\tilde{v}$ and $\tilde{q}$ are independent estimates of the value function, the total influence of trajectory $j$ is given by

\begin{align}
    I_j &= \frac{1}{N-1} \Bigg( \Bigg. \hat{v}_{DR} - \sum_{t=0}^{T-1} \gamma^t \Bigg( \Bigg. w_{0:t}^{(j)} r_t^{(n)}  \\ \nonumber
     &- w_{0:t}^{(j)} \tilde{q}(x^{(j)}_t, a^{(j)}_t) + w_{0:t-1}^{(j)} \tilde{v}(x^{(j)}_t) \Bigg. \Bigg) \Bigg. \Bigg).
\end{align}

\subsection{Influence of weighted IS estimators}

For weighted estimators such as weighted importance sampling (WIS) given by 

\begin{equation}
    \hat{v}^{\pi_e}_{WIS} = \frac{1}{\sum_{n=1}^N w_{0:T}^{(n)}}\sum_{n=1}^N w_{0:T}^{(n)} g_T^{(n)},
\end{equation}

the influence calculation is slightly different, and requires caching the sum of weights of all trajectories in the data.

\begin{align}
    I_j &= \hat{v}_{-j} - \hat{v} \\ \nonumber
    &= \frac{1}{\sum_{n \neq j} w_{0:T}^{(n)}} \sum_{n \neq j} w_{0:T}^{(n)} g_T^{(n)} \\ \nonumber
    &- \frac{1}{\sum_{n=1}^{N} w_{0:T}^{(n)}} \sum_{n=1}^{N} w_{0:T}^{(n)} g_T^{(n)} \\ \nonumber
    &= \sum_{n=1}^{N} w_{0:T}^{(n)} g_T^{(n)} \left( \frac{1}{\sum_{n \neq j} w_{0:T}^{(n)}} - \frac{1}{\sum_{n=1}^N w_{0:T}^{(n)}} \right) \\ \nonumber
    &- \frac{1}{\sum_{n \neq j} w_{0:T}^{(n)}} w_{0:T}^{(j)} g_T^{(j)} \\ \nonumber
    &= \frac{w_{0:T}^{(j)}}{(\sum_{n \neq j} w_{0:T}^{(n)})(\sum_{n=1}^N w_{0:T}^{(n)})} \sum_{n=1}^{N} w_{0:T}^{(n)} g_T^{(n)} \\ \nonumber
    &- \frac{1}{\sum_{n \neq j} w_{0:T}^{(n)}} w_{0:T}^{(j)} g_T^{(j)} \\ \nonumber
    &= \frac{w_{0:T}^{(j)}}{W - w_{0:T}^{(n)}} \left( \hat{v}_{WIS} - g_T^{(j)} \right).
\end{align}

In the last expression, $W = \sum_{n=1}^{N} w_{0:T}^{(n)}$ is the cached sum of weights.

\paragraph{WDR}

For the weighted doubly robust estimator (WDR) \citep{thomas2016data} the influence calculation is conceptually similar, but the fact that the sum of weights which normalizes the estimator is time dependant makes it a little more tedious and requires caching a number of values which scales with the horizon, $T$. The estimator is given by

\begin{align}
    \hat{v}_{WDR}^{\pi_e} &= \sum_{n=1}^N \Bigg( \Bigg. \sum_{t=0}^{T-1} \gamma^t \Bigg( \Bigg. \frac{w_{0:t}^{(n)}}{W_t} r_t^{(n)} \\ \nonumber
    &- \frac{w_{0:t}^{(n)}}{W_t} \tilde{q}(x^{(n)}_t, a^{(n)}_t) + \frac{w_{0:t-1}^{(n)}}{W_{t-1}} \tilde{v}(x^{(n)}_t) \Bigg. \Bigg) \Bigg. \Bigg),
\end{align}

where we define $W_t = \sum_{n=1}^N w_{0:t}^{(n)}$.

If we switch the order of summation and treat the three terms in the sum independently, we can think of the estimator as being composed of $3T$ terms:

\begin{align}
\label{eq:wdr_decomposition}
    \hat{v}_{WDR}^{\pi_e} = \sum_{t=0}^{T-1} \gamma^t &\Bigg( \Bigg. \sum_{n=1}^N \frac{w_{0:t}^{(n)}}{W_t} r_t^{(n)} \\ \nonumber
    &-\sum_{n=1}^N \frac{w_{0:t}^{(n)}}{W_t} \tilde{q}(x^{(n)}_t, a^{(n)}_t) \\ \nonumber
    &+\sum_{n=1}^N \frac{w_{0:t-1}^{(n)}}{W_{t-1}} \tilde{v}(x^{(n)}_t) \Bigg. \Bigg).
\end{align}

For a given $t$, let's look at the resulting difference in the first term if trajectory $j$ is removed from the data:

\begin{align}
\label{eq:wdr_term_1}
    &\sum_{n \neq j} \frac{w_{0:t}^{(n)}}{(W_t)_{-j}} r_t^{(n)} - \sum_{n=1}^N \frac{w_{0:t}^{(n)}}{W_t} r_t^{(n)} \\ \nonumber
    &= \sum_{n=1}^N w_{0:t}^{(n)} r_t^{(n)}  \left( \frac{1}{(W_t)_{-j}} - \frac{1}{W_t} \right) \\ \nonumber
    &- \frac{w_{0:t}^{(j)}}{(W_t)_{-j}} r_t^{(j)} \\ \nonumber
    &= \frac{w_{0:t}^{(j)}}{(W_t)_{-j}} \left( \sum_{n=1}^N \frac{w_{0:t}^{(n)} r_t^{(n)}}{W_t} -  r_t^{(j)} \right) \\ \nonumber
    &= \frac{w_{0:t}^{(j)}}{W_t - w_{0:t}^{j}} \left( A_t -  r_t^{(j)} \right),
\end{align}

where we defined $A_t = \sum_{n=1}^N \frac{w_{0:t}^{(n)} r_t^{(n)}}{W_t}$.

If we similarly define $B_t = \sum_{n=1}^N \frac{w_{0:t}^{(n)} \tilde{q}(x^{(n)}_t, a^{(n)}_t)}{W_t}$ and $C_t = \sum_{n=1}^N \frac{w_{0:t-1}^{(n)} \tilde{v}(x^{(n)}_t)}{W_{t-1}}$ and repeat the calculation in Equation \ref{eq:wdr_term_1} for the second and third terms in Equation \ref{eq:wdr_decomposition} (note the time offset in the definition of $C_t$) we see that the influence of the WDR is given by

\begin{align}
    I_j = \sum_{t=0}^{T-1} \gamma^t &\Bigg( \Bigg. \frac{w_{0:t}^{(j)}}{W_t - w_{0:t}^{j}} \left( A_t -  r_t^{(j)} \right) \\ \nonumber
    &- \frac{w_{0:t}^{(j)}}{W_t - w_{0:t}^{j}} \left( B_t -  \tilde{q}(x^{(j)}_t, a^{(j)}_t) \right) \\ \nonumber
    &+ \frac{w_{0:t-1}^{(j)}}{W_{t-1} - w_{0:t-1}^{j}} \left( C_t -  \tilde{v}(x^{(j)}_t) \right) \Bigg. \Bigg).
\end{align}.


\section{Preprocessing and experimental details for MIMIC-III acute hypotension dataset}
\label{appendix:mimic_details}

In this section, we describe the preprocessing we performed on the raw MIMIC-III database to convert it into a dataset amenable to modeling with RL. This preprocessing procedure was done in close consultation with the intensivist collaborator on our team.

\subsection{Cohort Selection}

We use MIMIC-III v1.4 \citep{johnson2016mimic}, which contains information from about 60,000 intensive care unit (ICU) admissions to Beth Israel Deaconess Medical Center. We filter the initial database on the following features: admissions where data was collected using the Metavision clinical information system; admissions to a medical ICU (MICU); adults (age $\geq$ 18 years); initial ICU stays for hospital admissions with multiple ICU stays; ICU stays with a total length of stay of at least 24 hours; and ICU stays where there are 7 or more mean arterial pressure (MAP) values of 65mmHg or less, indicating probable acute hypotension. For long ICU stays, we limit to only using information captured during the inital 48 hours after admission, as our intensivist advised that care for hypotension during later periods of an ICU stay often look very different. After this filtering, we have a final cohort consisting of 1733 distinct ICU admissions. For computational convenience, we further down-sample this cohort, and use 20\% (346) ICU stays to use to learn a policy, and another 20\% (346) ICU stays to evaluate the policy via FQE and our proposed influence analysis. 

\subsection{Clinical Variables Considered}

Given our final cohort of patients admitted to the ICU, we next discuss the different clinical variables that we extract that are relevant to our task of acute hypotension management. 

The two first-line treatments are intravenous (IV) fluid bolus therapy, and vasopressor therapy. We construct fluid bolus variables in the following way:
\begin{enumerate}
    \item We filter all fluid administration events to only include NaCl 0.9\%, lactated ringers, or blood transfusions (packed red blood cells, fresh frozen plasma, or platelets).
    \item Since a fluid bolus should be a nontrivial amount of fluid administered over a brief period of time, we further filter to only fluid administrations with a volume of at least 250mL and over a period of 60 minutes or shorter.
\end{enumerate}
Each fluid bolus has an associated volume, and a starting time (since a bolus is given quickly / near-instantaneously, we ignore the end-time of the administration). To construct vasopressors, we first normalize vasopressor infusion rates across different drug types as follows, using the same normalization as in \citet{komorowski2018artificial}:
\begin{enumerate}
    \item Norepinephrine: this is our ``base'' drug, as it's the most commonly administered. We will normalize all other drugs in terms of this drug. Units for vasopressor rates are in mcg per kg body weight per minute for all drugs except vasopressin.
    \item Vasopressin: the original units are in units/min. We first clip any values above 0.2 units/min, and then multiply the final rates by 5.
    \item Phenylephrine: we multiply the original rate by 0.45.
    \item Dopamine: we multiply the original rate by 0.01.
    \item Epinephrine: this drug is on the same scale as norepinephrine and is not rescaled.
\end{enumerate}
As vasopressors are given as a continuous infusion, they consist of both a treatment start time and stop time, as well as potentially many times in the middle where the rates are changed. More than a single vasopressor may be administered at once, as well.

We also use 11 other clinical variables as part of the state space in our application: serum creatinine, FiO$_2$, lactate, urine output, ALT, AST, diastolic/systolic blood pressure, mean arterial pressure (MAP; the main blood pressure variable of interest), PO$_2$, and the Glasgow Coma Score (GCS). 

\subsection{Selecting Action Times}

Given a final cohort, clinical variables, and treatment variables, we still must determine how to discretize time and choose at which specific time points actions should be chosen. To arrive at a final set of ``action'' times for a specific ICU stay, we use the following heuristic-based algorithm:
\begin{enumerate}
    \item We start by including all times a treatment is started, stopped, or modified.
    \item Next, we remove consecutive treatment times if there are no MAP measurements between treatments. We do this because without at least one MAP measurement in between treatments, we would not be able to assess what effect the treatment had on blood pressure. This leaves us with a set of time points when treatments were started or modified. 
    \item At many time points, the clinician consciously chooses not to take an action. Unfortunately, this information is not generally recorded (although, on occasion, may exist in clinical notes). As a proxy, we consecutively add to our existing set of ``action times'' any time point at which an abnormally low MAP is observed ($<60$mmHg) and there are no other ``action times'' within a 1 hour window either before or after. This captures the relatively fine-granularity with which a physician may choose not to treat despite some degree of hypotension.
    \item Last, we add additional time points to fill in any large gaps where no ``action times'' exist. We do this by adding time points between existing ``action times'' until there are no longer any gaps greater than 4 hours between actions. This makes some clinical sense, as patients in the ICU are being monitored relatively closely, but if they are more stable, their treatment decisions will be made on a coarser time scale.
\end{enumerate}

Now that we have a set of action times for each trajectory, we can count up the total number of transitions in our training and evaluation datasets (both of which consist of 346 trajectories). The training trajectories contain a total of 6777 transitions, while there are 6863 total transitions in the evaluation data. Trajectories vary in length from a minimum of 7 transitions to a maximum of 49, with 16, 18, and 23 transitions comprising the 25\%, 50\%, and 75\% quantiles, respectively.

\subsection{Action Space Construction}

Given treatment timings, doses, and manually identified ``action times'' at which we want to assess what type of clinical decision was made, we can now construct our action space. We choose to operate in a discrete action space, which means we need to decide how to bin each of the continuous-valued treatment amounts.  

Binning of IV fluids is more natural and easier, as fluid boluses are generally given in discrete amounts. The most common bolus sizes are 500mL and 1000mL, so we bin fluid bolus volumes into the following 4 bins, which correspond to ``none''/``low''/``medium''/``high'' (in mL): $\{0, [250,500), [500,1000), [1000,\infty]\}$, although in practice very few boluses of more than 2L are ever given. Given this binning scheme, we can simply add up the total amount of fluids administered during any adjacent action times to determine which discrete fluid amount we should code the action as.

Binning of vasopressors is slightly more complex. These drugs are dosed at a specific rate, and there may be many rate changes made between action times, or sometimes there are several vasopressors being given at once. We chose to first add up the cumulative amount of (normalized) vasopressor drug administered between action times, and then normalize this amount by the size of the time window between action times to account for the irregular spacing. Finally, we also bin vasopressors into 4 discrete bins corresponding to ``none''/``low''/``medium''/``high'' amounts: $\{0, (0,8.1), [8.1, 21.58), [21.58, \infty]\}$. The relevant units here are total mcg of drug given each hour, per kg body weight. Since the distribution of values for vasopressors is not as naturally discrete, we chose our bin sizes using the 33.3\% and 66.7\% quantiles of dose amounts.

In the end, we have an action space with 16 possible discrete actions, considering all combinations of each of the 4 vasopressor amounts and fluid bolus amounts.

\subsection{State Construction}

Given a patient cohort, decision/action times, and discrete actions, we are now ready to construct a state space. For simplicity in this initial work, we first start with the 11 clinical time series variables previously listed. If a variable is never measured, we use the population median as a placeholder. If a variable has been measured before, we use the most recent measurement. The sole exception to this is the 3 blood pressure variables. For the blood pressures, we instead use the minimum (or worst) value observed since the last action.

We add to these a number of indicator variables that denote whether a particular variable was recently measured or not. Due to the strongly missing-not-at-random nature of clinical time series, there is often considerable signal in knowing that certain types of measurements were recently taken, irrespective of the measurement values \citep{agniel2018biases}. We choose to construct indicator variables denoting whether or not a urine output was taken since the last action time, and whether a GCS was recorded since the last action. We also include state features denoting whether the following labs/vitals were \textit{ever} ordered: creatinine, FiO$_2$, lactate, ALT, AST, PO$_2$. We do not include these indicators for all 11 clinical variables, as most of the vitals are recorded at least once an hour, and sometimes even more frequently. In total, 8 indicators comprise part of our state space.

Last, we include 10 additional variables that summarize past treatments administered, if any. We first include 6 indicator variables (3 for each treatment type) denoting which dose of fluid and vasopressor, if any, was chosen at the last action time. Last, for each treatment type we include two final features that summarize past actual amounts of treatments administered (the total amount of this treatment administered up until the current time, and the total amount of this treatment administered within the last 8 actions.

In total, our final state space has 29 dimensions. In future work we plan to explore richer state representations.

\subsection{Reward Function Construction}

In this preliminary work, we use a simple reward that is a piecewise linear function of the MAP in the next state. In particular, the reward takes on a value of $-1$ at 40mmHg, the lowest attainable MAP in the data. It increases linearly to -0.15 at 55mmHg, linearly from there to -0.05 at 60mmHg, and achieves a maximum value of 0 at 65mmHg, a commonly used target for blood pressure in the ICU \citep{asfar2014high}. However, if a patient has a urine output of 30mL/hour or higher, then any MAP values of 55mmHg or higher are reset to 0. This attempts to mimic the fact that a clinician will not be too concerned if a patient is slightly hypotensive but otherwise stable, since a modest urine output indicates that the modest hypotension is not a real problem.

\subsection{Choice of Kernel Function}

In order to use kernel-based FQE, we need to define a kernel that defines similarity between states. In consultation with our intensivist collaborator, we chose a simple weighted Euclidean distance, where each state variable receives a different weight based on its estimated importance to the clinical problem. We show all weights in Table \ref{table:kernel-func}.

\begin{table}
\caption{Weights for each state variable in our kernel function.}
\label{table:kernel-func}

\begin{tabular}{|l|l|}
\hline
\textbf{State Variable} & \textbf{Kernel Weight} \\ \hline
Creatinine & 3 \\ \hline
FiO2 & 15 \\ \hline
Lactate & 10 \\ \hline
Urine Output & 15 \\ \hline
\begin{tabular}[c]{@{}l@{}}Urine Output \\ since last action?\end{tabular} & 5 \\ \hline
ALT & 5 \\ \hline
AST & 5 \\ \hline
Diastolic BP & 5 \\ \hline
MAP & 15 \\ \hline
PO2 & 3 \\ \hline
Systolic BP & 5 \\ \hline
GCS & 15 \\ \hline
\begin{tabular}[c]{@{}l@{}}GCS since \\ last action?\end{tabular} & 5 \\ \hline
\begin{tabular}[c]{@{}l@{}}Creatinine ever\\ taken?\end{tabular} & 3 \\ \hline
\begin{tabular}[c]{@{}l@{}}FiO2 ever \\ taken?\end{tabular} & 15 \\ \hline
\begin{tabular}[c]{@{}l@{}}Lactate ever\\ taken?\end{tabular} & 10 \\ \hline
ALT ever taken? & 5 \\ \hline
AST ever taken? & 5 \\ \hline
PO2 ever taken? & 3 \\ \hline
\begin{tabular}[c]{@{}l@{}}Low vasopressor\\ done last time?\end{tabular} & 15 \\ \hline
\begin{tabular}[c]{@{}l@{}}Medium vasopressor\\ done last time?\end{tabular} & 15 \\ \hline
\begin{tabular}[c]{@{}l@{}}High vasopressor\\ done last time?\end{tabular} & 15 \\ \hline
\begin{tabular}[c]{@{}l@{}}Low fluid done\\ last time?\end{tabular} & 15 \\ \hline
\begin{tabular}[c]{@{}l@{}}Medium fluid done\\ last time?\end{tabular} & 15 \\ \hline
\begin{tabular}[c]{@{}l@{}}High fluid done\\ last time?\end{tabular} & 15 \\ \hline
\begin{tabular}[c]{@{}l@{}}Total vasopressors\\ so far\end{tabular} & 15 \\ \hline
Total fluids so far & 15 \\ \hline
\begin{tabular}[c]{@{}l@{}}Total vasopressors\\ last 8 actions\end{tabular} & 15 \\ \hline
\begin{tabular}[c]{@{}l@{}}Total fluids \\ last 8 actions\end{tabular} & 15 \\ \hline
\end{tabular}
\end{table}

Since technically we need a kernel over both all possible states and actions for FQE and influence analysis, we augment our kernel with extremely large weights so that effectively the kernel only compares pairs $(s,a)$ and $(s',a')$ for $a=a'$. Other choices should be made for continuous action spaces.

\subsection{Hyperparameters}

We use the training set of 6777 trajectories to learn a policy to then evaluate using FQE and influence analysis. In particular, we learn a deterministic policy by taking the most common action within the 50 nearest neighbors of a given state, with respect to the kernel in Table \ref{table:kernel-func}. We use a discount of $\gamma=1$ so that all time steps are treated equally, and use a neighborhood radius of 7 for finding nearest neighbors in FQE. Lastly, for the influence analysis, we use a threshold of 0.05, or 5\%, so that transitions which will affect the FQE value estimate by more than 5\% are flagged for expert review.

\section{Additional Results from MIMIC-III acute hypotension dataset}
\label{sec:additional-mimic-results}

In the main body of the paper, we showed two qualitative results figures showing 2 of the 6 highly influential transitions flagged by influence analysis. In this section, we show the remaining 4 influential transitions. 

\begin{figure}[H]
\centering
\includegraphics[width=0.45\textwidth]{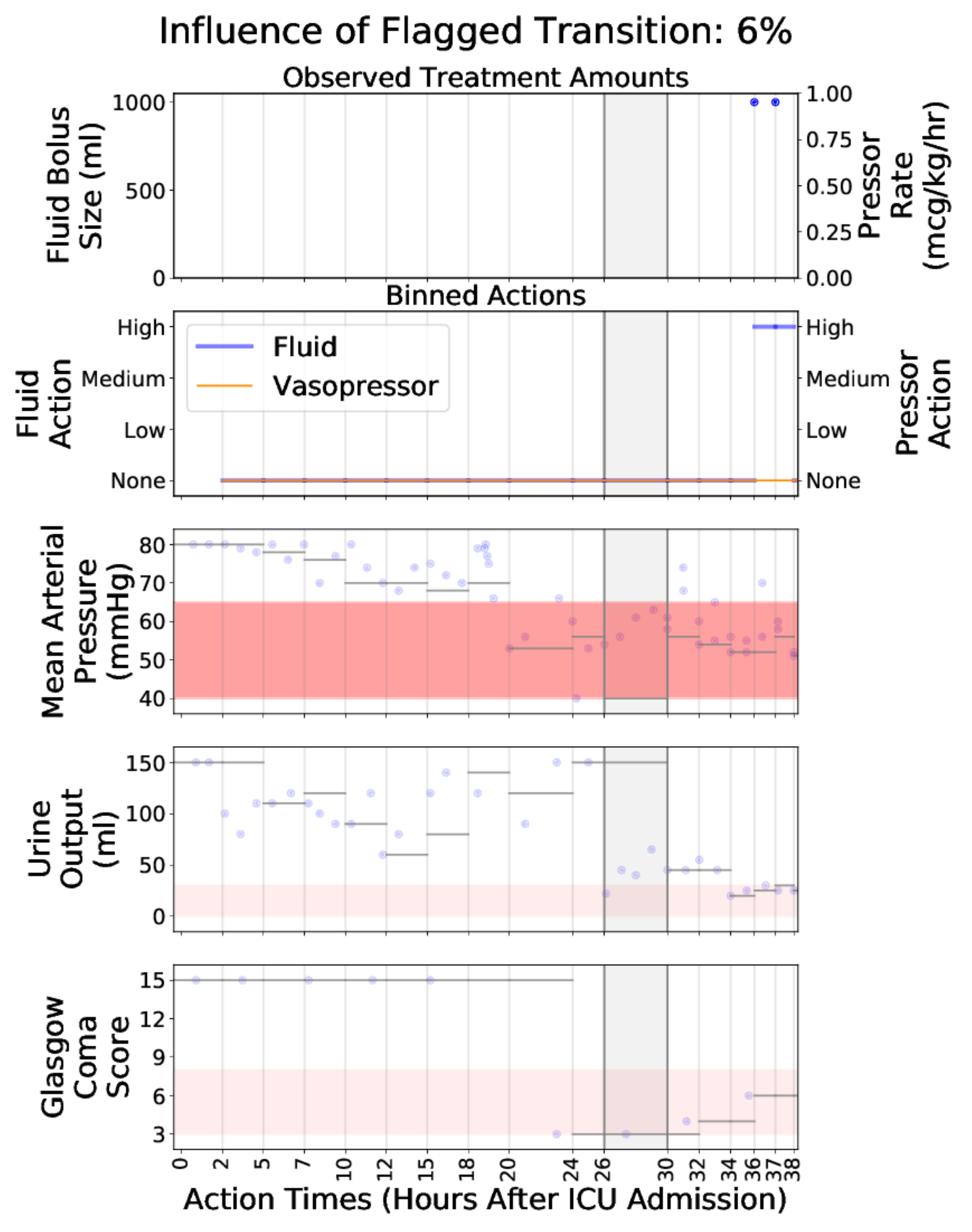}
\caption{An additional example identified by our influence analysis as having an especially high effect on the OPE value estimate.} 
\label{fig:mimic-9}
\end{figure}

\begin{figure}[H]
\centering
\includegraphics[width=0.45\textwidth]{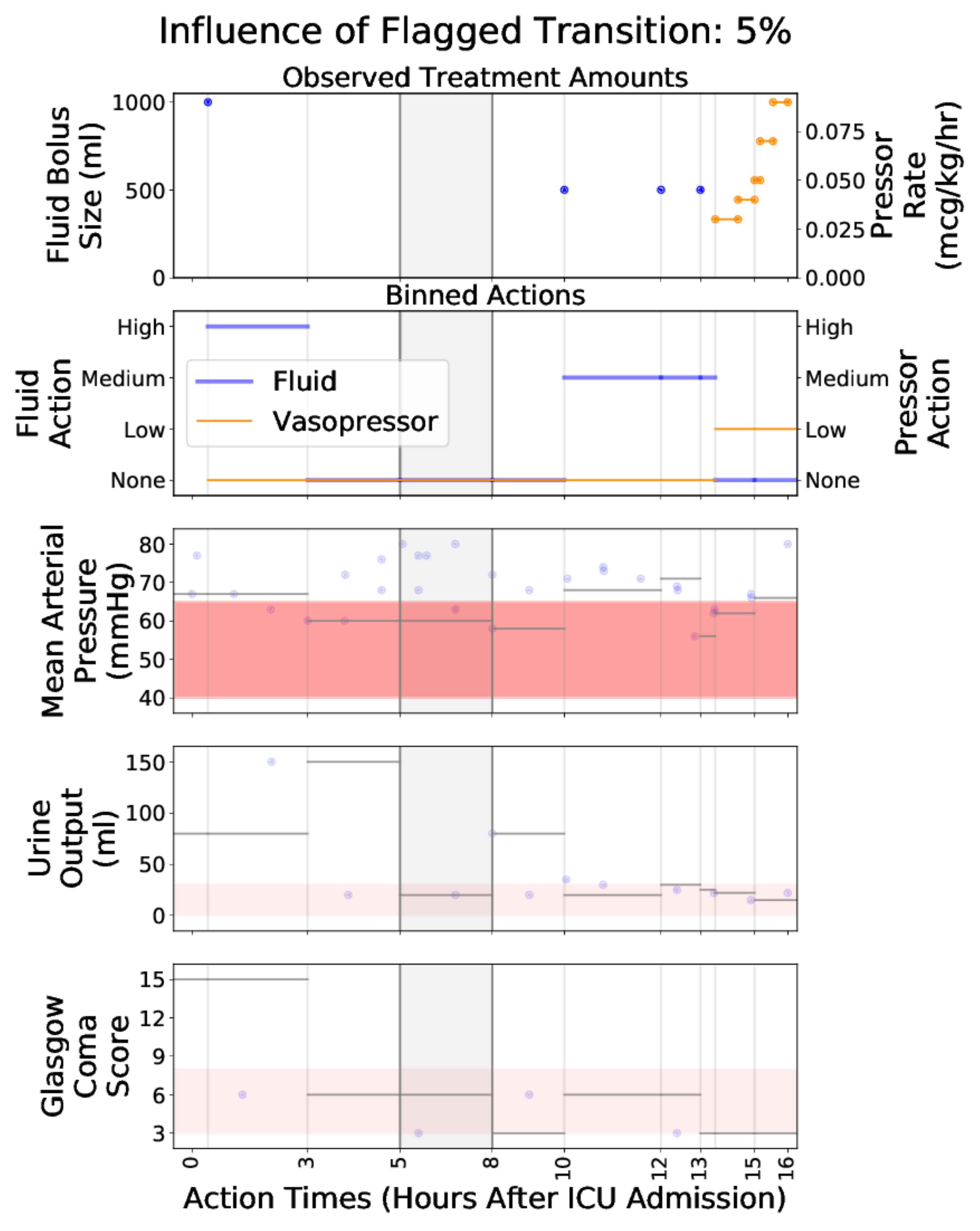}
\caption{An additional example identified by our influence analysis as having an especially high effect on the OPE value estimate. Note that this transition is from the same trajectory as the influential transition highlighted in Figure \ref{fig:mimic-3}} 
\label{fig:mimic-2}
\end{figure}

\begin{figure}[H]
\centering
\includegraphics[width=0.45\textwidth]{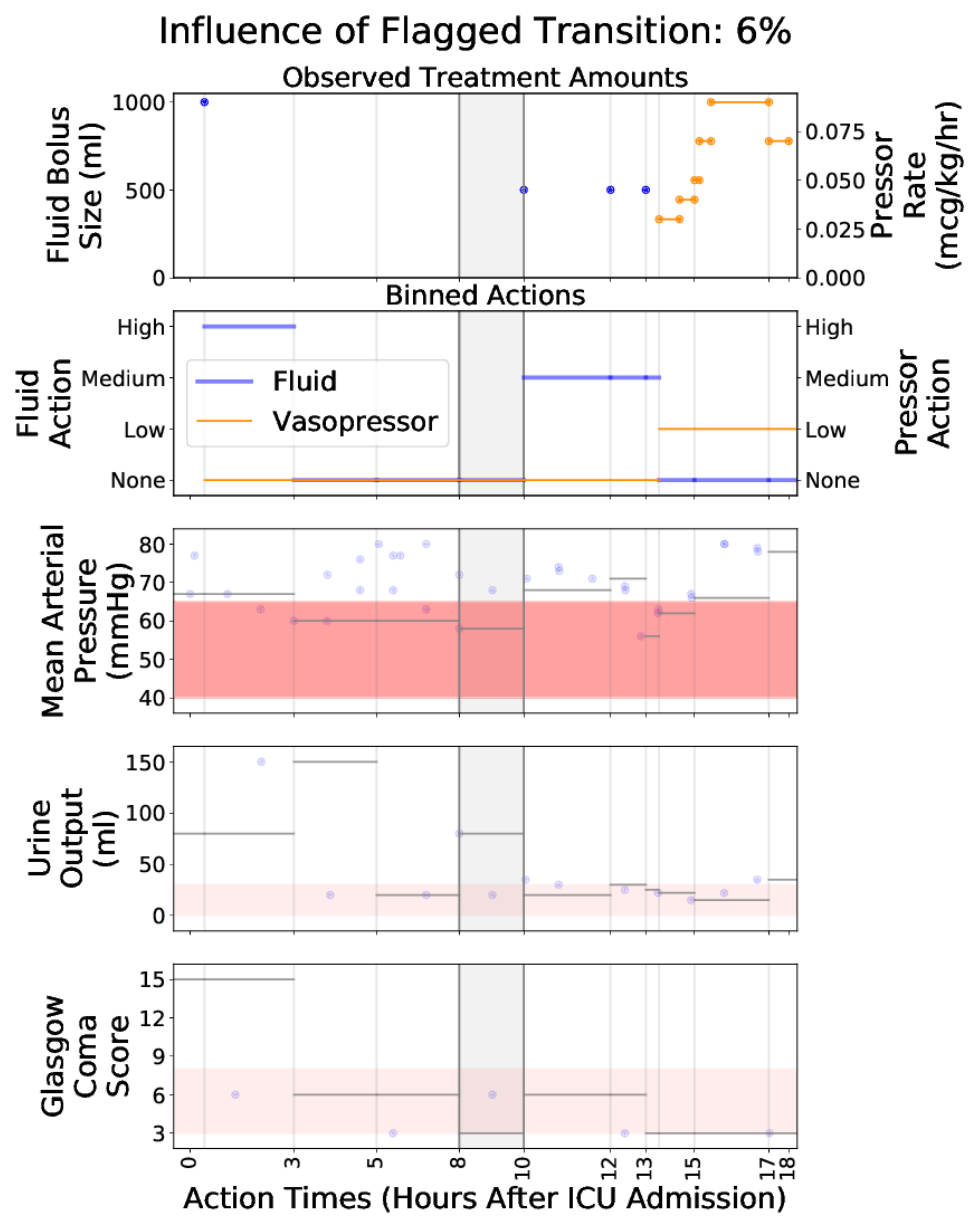}
\caption{An additional example identified by our influence analysis as having an especially high effect on the OPE value estimate. Note that this transition is from the same trajectory as the influential transition highlighted in Figure \ref{fig:mimic-2}} 
\label{fig:mimic-3}
\end{figure}

\begin{figure}[H]
\centering
\includegraphics[width=0.45\textwidth]{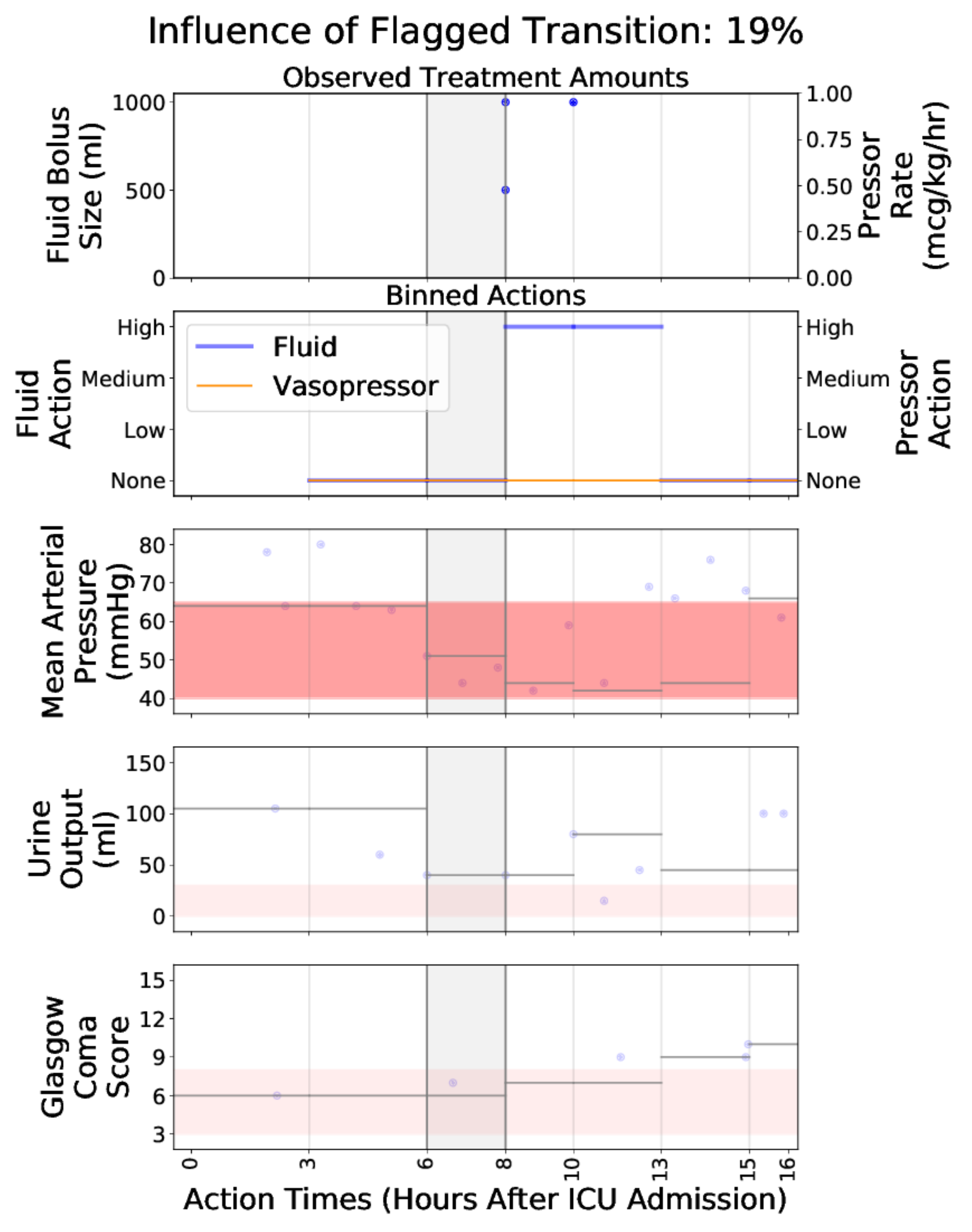}
\caption{An additional example identified by our influence analysis as having an especially high effect on the OPE value estimate.} 
\label{fig:mimic-1}
\end{figure}

\end{document}